\documentclass[twoside,11pt]{article}
\usepackage{jmlr2e1}
\usepackage[colorlinks=true, citecolor=blue]{hyperref}

\usepackage{times}
\usepackage{graphicx} 
\usepackage{subfigure} 

\usepackage{natbib}

\usepackage{algorithm}
\usepackage{algorithmic}

\usepackage{amssymb,amsmath,amsfonts}
\usepackage{mathrsfs}
\usepackage{wasysym}
\newcommand{\norm}[1]{\left\lVert#1\right\rVert}
\newcommand{\expect}[1]{\mathbb{E}\left[{#1}\right]}

\newcommand{\given}{\; \big\vert \;} 
\newcommand{\bydef}{:=}
\newcommand{\OFUL}{\textbf{OFUL}}

\usepackage{subfigure}
\usepackage{mathnotations2}

\graphicspath{{./Figures/}}

\newtheorem{mytheorem}{Theorem}
\newtheorem{myproposition}{Proposition}
\newtheorem{mylemma}{Lemma}
\newtheorem{mycorollary}{Corollary}

\newtheorem{assumption}{Assumption}

\newcommand{\beq}{\begin{equation}}
\newcommand{\eeq}{\end{equation}}
\newcommand{\beqa}{\begin{eqnarray}}
\newcommand{\eeqa}{\end{eqnarray}}
\newcommand{\beqan}{\begin{eqnarray*}}
\newcommand{\eeqan}{\end{eqnarray*}}

\firstpageno{1}

\begin{document} 

\title{Low-rank Bandits with Latent Mixtures}
\author{\name{Aditya Gopalan}$^1$ \email{aditya@ece.iisc.ernet.in}\\    \name{Odalric-Ambrym Maillard}$^2$ \email{odalricambrym.maillard@inria.fr}\\ 
\name{Mohammadi Zaki}$^1$ \email{zaki@ece.iisc.ernet.in}\\ 
\addr     $1.$  Electrical Communication Engineering,\\
Indian Institute of Science,\\
Bangalore 560012, India\\
\addr
                $2.$ Inria Saclay -- \^Ile de France\\
                Laboratoire de Recherche en Informatique \\
                660 Claude Shannon\\
                91190 Gif-sur-Yvette, France\\
}


\maketitle
%
%
%

\begin{abstract} 
  We study the task of maximizing rewards from recommending
  items (actions) to users sequentially interacting with a recommender
  system. Users are modeled as latent mixtures of $C$ many
  representative user classes, where each class specifies a mean
  reward profile across actions. Both the user features (mixture
  distribution over classes) and the item features (mean reward vector per
  class) are unknown a priori. The user identity is the only contextual information available to the learner while interacting. This induces a low-rank structure on the matrix of
  expected rewards $r_{a,b}$ from recommending item $a$ to user $b$.
  The problem reduces to the well-known linear bandit when either
  user- or item-side features are perfectly known. In the setting
  where each user, with its stochastically sampled taste profile,
  interacts only for a small number of sessions, we develop a   bandit algorithm for the two-sided uncertainty. It combines the
  Robust Tensor Power Method of
  \citet{Anandkumar} with the \OFUL\ linear bandit algorithm of
  \cite{AbbPalCsa11:linbandits}. We provide the first rigorous regret
  analysis of this combination, showing that its regret
  after $T$ user interactions is $\tilde O(C\sqrt{BT})$, with
  $B$ the number of users. An ingredient towards this result
  is a novel robustness property of \OFUL{}, of independent interest.
\end{abstract} 

\section{{Introduction}}

Recommender systems aim to provide targeted, personalized content
recommendations to users by learning their responses over time. The
underlying goal is to be able to predict which items a user might
prefer based on preferences expressed by other related users and
items, also known as the principle of collaborative filtering.

A popular approach to model preferences expressed by users in
recommender systems is via probabilistic mixture models or {\em latent
  class} models \citep{HofPuz99:latent,KleSan04:mixturemodels}. In
such a mixture model, we have a set of $A$ items (content) that can be
recommended to $B$ users (consumers). Whenever item $a$ is recommended
to user $b$, the system gains an expected reward of $r_{a,b}$. The key
structural assumption that captures the relationship between users'
preferences is that there exists a set of latent set of $C$ {\em
  representative} user types or typical taste profiles. Formally, each
taste profile $c$ is a unique vector
${\bf u}_c \equiv \left(u_{a,c}\right)_{a}$ of the expected rewards
that every item $a$ elicits under the taste profile. Each user $b$ is
assumed to sample one of the typical profiles randomly using an
individual probability distribution
${\bf v}_{b} \equiv \left(v_{b,c}\right)_c$; its reward distribution
across the items subsequently becomes that induced by the assumed
profile.

Our focus is to address the sequential optimization of net reward
gained by the recommender, {\em without any prior knowledge of either
  the latent user classes or users' mixture distributions}. Assuming
that users arrive to the system repeatedly following an unknown
stochastic process and re-sample their profiles over time, according
to their respective {\em unknown} mixtures across latent classes, we
seek online learning strategies that can achieve {\em low regret
  relative to the best single item that can be recommended to each
  user}. Note that this is qualitatively different than the task of
estimating latent classes or user mixtures in a batch fashion,
well-studied by now \citep{SutTenSal09:Bayescluster,
  AnaGeHsuKak14:tensormixedmodels,Anandkumar}; the task of
simultaneously optimizing net utility in a bandit fashion in complex
expression models like these has received little or no analytical
treatment. Our work takes a step towards filling this void.

An especially challenging aspect of online learning in recommender
systems is the relatively meager number of available interactions with
a same user, which is offset to an extent by the assumption that users
can only have a limited number of taste profiles (classes). Indeed, if
one can identify the class to which a certain user belongs and
aggregate information from all other users in that class, then one can
recommend to the user the best item for the class. In practice,
classes are latent and not necessarily known in advance, and several
works
\citep{gentile2014online,lazaric2013sequential,maillard2014latent}
study the restricted situation when each user always belongs to one
specific class (i.e., when all mixture distributions have
support size $1$). We go two steps further, since in many situations
(a) users cannot be assumed to belong to one class only, such as when
a user account is shared by several individuals (e.g. a smart-TV), and
(b) the duration of a user-session, that is the number of consecutive
recommendations to the same individual connected to a user-account,
cannot assumed to be long\footnote{It is also unlikely to be very
  short, say, less than 3.}.

The key challenges that this work addresses are (1) the lack of
knowledge of ``features'' on {\em both} the user-side and item-side in
a linear bandit problem (in this case, both the user mixture weights
and the item class reward profiles) and (2) provable regret
minimization with very few i.e. $O(1)$ interactions with every user
$b$ having a specific taste profile, as opposed to a large number of
interactions such as in transfer learning
\citep{lazaric2013sequential}.


{\bf Contributions and overview of results.} We consider a setting
when users are assumed to come from arbitrary mixtures across classes
(they are not assumed to fall perfectly in one class as was the
assumption in works by
\citet{gentile2014online,maillard2014latent}). {\em We develop a novel
  bandit algorithm (Algorithm~\ref{alg:PerOFULwithExploration}) that
  combines (a) the Optimization in the Face of Uncertainty Linear
  bandit \OFUL{} algorithm \citep{AbbPalCsa11:linbandits} for bandits
  with known action features, and (b) a variant of the Robust Tensor
  Power (RTP) algorithm \citep{Anandkumar} that uses only bandit
  (partial) estimates of latent user classes with observations coming
  from a mixture model.} More specifically, we introduce a subroutine
(Algorithm~\ref{alg:MainIntuitive}) that makes use of the RTP method
to extract item-side attributes ($U$) and, contributing to its
theoretical analysis, show a recovery property
(Theorem~\ref{thm:Uestimates}). Note that the RTP method ideally
requires (unbiased) estimates of the $2$nd and $3$rd order moments of
actions' rewards, but with bandit information the learner can access
only partial reward information, i.e., a single reward sample from an
action. To overcome this, we devise an importance sampling scheme
across $3$ successive time instants to build the $2$nd and $3$rd order
moment tensor estimates that RTP uses. For the task of issuing
recommendations, we develop an algorithm
(section~\ref{sec:featuresV}), essentially based on \OFUL{},
instantiated per user, using for each $a$ the {\em estimated} latent
class vectors $\{u_{a,c}\}_c$ (obtained via the RTP subroutine) as arm
features, and uncertain parameter vector to be learned ${\bf v}_b$.

We carry out a rigorous analysis of the algorithm and show that it
achieves regret $\tilde{O}(\ell C \sqrt{BT})$ in $T$ rounds of
interaction (Theorem~\ref{thm:main}), provided each arriving user
interacts with the system for $\ell \geq 3$ rounds with the same
profile. In comparison, the regret of the strategy that completely
disregards the latent mixture structure of rewards and employs a
standard bandit strategy (e.g. UCB \citep{AuerEtAl02FiniteTime}) per
user, scales as $O(B \sqrt{TA/B}) = O(\sqrt{ABT})$ after $T$
rounds\footnote{Roughly, each UCB
  per-user plays from a pool of $A$ actions for about $T/B$ rounds,
  thus suffering regret $O(\sqrt{A(T/B)})$.} , which is considerably
suboptimal in the practical case with a very large number of items but
very few representative user classes ($C \ll A$).  It is also worth
noting that the regret bound we achieve, order-wise, is what would
result from applying the \OFUL\ or any optimal linear bandit algorithm
assuming \textit{a priori knowledge} of all latent user classes
$\{u_{a,c}\}_{a,c}$, that is $\tilde{O}(\ell C \sqrt{BT})$. In this
sense, our result shows that one can {\em simultaneously estimate
  features on both sides of a bilinear reward model and achieve regret
  performance equivalent to that of a one-sided linear model}, which
is the first result of its kind to the best of our
knowledge\footnote{An earlier result of \citet{DjoKraCev13:GPbandits}
  gets $O(T^{4/5})$ regret while moreover assuming a perfect control
  of the sampling process (we can't assume this due to the user
  arrivals).}. Our results are presented for finite time horizons with
explicit details of the constants arising from the error analysis of
RTP, which at this point are large but possibly improvable.

  En route to deriving the regret for our algorithm, we also make
  a novel contribution that advances the theoretical understanding of
  \OFUL, and which is of independent interest. We show that in the
  standard linear bandit setting, where the expected reward of an arm
  linearly depends on $d$ features, {\em \OFUL{} yields (sub-linear)
  $\tilde{O}\left(\rho d \sqrt{T}\right)$ regret even when it makes
  decisions based on {\em perturbed} or inexact feature vectors}
  (Theorem~\ref{thm:pOFUL2}), where $\rho$ quantifies the
  distortion. This property holds whenever the perturbation error is
  small enough, and we explicitly give both (a) a sufficient condition
  on the size of the perturbation in terms of the set of actual
  features, and (b) a bound on the (multiplicative) distortion $\rho$
  in the regret due to the perturbation (note that $\rho = 1$ in the
  ideal linear case).
%
%
%

\section{{Setup and notation}}

For any positive integer $n$, $[n]$ denotes the set
$\{1, 2, \ldots, n\}$.

At each $n\in \Nat$, nature selects a user $b_n\in[B]$ according to the
probability distribution $\beta$ over $[B]$, independent of the past,
and $b_n$ is revealed to the learner.  A user class $c_{n}$ is
subsequently sampled from the probability distribution ${\bf v}_{b_n}$
over $[C]$, and $c_n$ (the assumed class of user $b_n$) interacts with
the learner for the next $\ell\geq 3$ consecutive steps. Such an
interaction will often be termed a {\em mini-session}.

In each step $l \in [\ell]$ of a mini-session, the learner plays an
action (issues a recommendation) $a_{n,l}\in [A]$ and subsequently
receives reward $Y_{n,l} = u_{a_{n,l},c_n} +\eta_{n,l}$, where
$\eta_{n,l}$ is a (centered) $R$-sub-Gaussian i.i.d. random variable
independent from $a_{n,l},c_n$, representing the  noise in the reward. We
let ${\bf u}_a \in \mathbb{R}^C$ represent the vector
$(u_{a,c})_{c \in [C]}$ of the mean rewards from action $a$ in each
class. Note that
$\Esp[u_{a_{n,l},c_n} | a_{n,l}] = \Esp[{\bf u}_{a_{n,l}}^\top {\bf
  v}_{b_n} | a_{n,l}]$.
For convenience, we use the index notation $t\equiv(n,l)$ and
introduce $T=N\ell$, where $N$ is the total number of mini-sessions,
and $T$ the total number of interactions of the learner with the
system. We denote likewise $Y_t,a_t,c_t,\eta_t$ for $Y_{n,l}$,
$a_{n,l}$, $c_n$, $\eta_{n,l}$, and let
$u_{\max} \eqdef \max_{a\in[A],c\in[C]}|u_{a,c}|$.

We are interested in designing an online recommendation strategy,
i.e., one that plays actions depending on past observations, achieving
low {\em (cumulative) regret} after $T \equiv (N,\ell)$ mini-sessions,
defined as $\cR_T \eqdef \sum_{n \in [N], l \in [\ell]} r_{n,l}$,
where
$r_{n,l} \eqdef \max_{a \in [A]} {\bf u}_a^\top {\bf v}_{b_n} - {\bf
  u}_{a_{n,l}}^\top {\bf v}_{b_n}$.
In other words, we wish to compete against a strategy that plays for
every user an action yielding the highest reward in expectation under
its mixture distribution over user classes.

\section{{Recovering latent user classes: The
    EstimateFeatures subroutine}}\label{sec:featuresU}

In this section, we provide an estimation algorithm for the matrix $U$, using the RTP method.\footnote{We consider
  $\ell = 3$ to describe the algorithm; $\ell>3$ is easily handled by
  repeating the 3-wise sampling $p(a,a',a'')$ for
  $\lfloor \ell/3 \rfloor$ times and discarding the remaining ($< 3$)
  steps in the mini-session during exploration (leading to a
  negligible regret overhead).}

{\bf Estimation of tensors.}  We assume that in mini-session $n$, when
interacting with user $b_n$, the triplet $\{a_{n,l}\}_{l\leq \ell}$ is
chosen from a distribution $p_n(a,a',a''|b_n)$. Letting
$X_{a_{n,l},b_n,n,l} \eqdef Y_{n,l} = u_{a_{n,l},c_n} +\eta_{n,l}$ to
explicitly indicate the active user and action chosen at $(n,l)$, we
form the importance-weighted estimates
\begin{align*}
  &\tilde r_{a,a',n} \eqdef  \frac{1}{n}\sum_{i=1}^n \frac{X_{a_{i,1},b_i,i,1}X_{a_{i,2},b_i,i,2}}{p_i(a,a'|b_i)} \indic{a_{i,1}=a,a_{i,2}=a'},\\
  &\tilde r_{a,a',a'',n} \eqdef  \frac{1}{n}\sum_{i=1}^n \frac{X_{a_{i,1},b_i,i,1}X_{a_{i,2},b_i,i,2}
    X_{a_{i,3},b_i,i,3}}{p_i(a,a',a''|b_i)}
  \indic{a_{i,1}=a,a_{i,2}=a', a_{i,3}=a''}\,.
\end{align*}

\noindent
for the second and third-order tensors\footnote{An alternative is the
  \textit{implicit exploration} method due to
  \citet{kocak2014efficient}.}.

We introduce the matrices
$\hat M_{n,2} \equiv (\tilde r_{a,a',n})_{a,a'\in[A]}$ and
$M_{2} \equiv (m_{a,a'})_{a,a'\in[A]}$ with
$m_{a,a'} \eqdef \Esp[\tilde r_{a,a',n} ]$, and the tensors
$\hat M_{n,3} \equiv (\tilde r_{a,a',a'',n})_{a,a',a''\in[A]}$ and
$M_{3} \equiv (m_{a,a',a''})_{a,a',a''\in[A]}$ with
$m_{a,a',a''} \eqdef \Esp[ \tilde r_{a,a',a'',n} ]$.  The following
result decomposes the matrix $M_2$ and
tensor $M_3$ as weighted sums of outer products.

\begin{mylemma}{}\label{lem:Mestimates}
  When the user arrivals are i.i.d. according to the law $\beta$,
  i.e., $b_i \stackrel{i.i.d}{\sim} \beta \; \forall i\in [n]$, it holds
  that
  
  \vspace{-6mm}
  \begin{align*}
    m_{a,a',n} & = \sum_{c\in [C]} v_{\beta,c} u_{a,c}u_{a',c}, \quad \text{ and }\\
    m_{a,a',a'',n} & =  \sum_{c\in [C]} v_{\beta,c} u_{a,c}u_{a',c}u_{a'',c}\,.
  \end{align*}
\end{mylemma}

Having shown the unbiasedness of the empirical $2$nd and $3$rd moment
tensors $\hat M_{n,2}$ and $\hat M_{n,3}$, we next turn to showing
concentration to their respective means.
\begin{mylemma}\label{lem:estimates}
	Assuming that $p_i(a,a'|b_i) \geq q_{2,i}$ and $p_i(a,a',a''|b_i)\geq q_{3,i}$ for deterministic $q_{2,i},q_{3,i}$, for all $i\in\Nat,a,a',a''\in[A]$, then
	for all $n\leq N$, with probability higher than $1-\delta$, it holds simultaneously for all $a,a',a''$ that

	\vspace{-8mm}
        \begin{align*}
          |\tilde r_{a,a',n} - m_{a,a',n}| &\leq  \sqrt{\sum_{i=1}^nq_{2,i}^{-2}\frac{\log(4A^2/\delta)}{2n^2}}, \\
	|\tilde r_{a,a',n} - m_{a,a',a'',n}| &\leq \sqrt{\sum_{i=1}^nq_{3,i}^{-2}\frac{\log(4A^3/\delta)}{2n^2}}.
        \end{align*}
\end{mylemma}

An immediate corollary is the following one:
\begin{mycorollary}\label{cor:estimates}
Provided that $q_{2,i} = \gamma_i/A^2$ and $q_{3,i} = \gamma_i/A^3$ for some $\gamma_i>0$, then on an event of probability higher than $1-\delta$, the following hold simultaneously:

\vspace{-6mm}
\beqan
e_n^{(2)} &\eqdef&	\|\hat M_{n,2}  - M_{2}\| \leq A^3\sqrt{\sum_{i=1}^n\gamma_i^{-2}\frac{\log(4A^2/\delta)}{2n^2}},
\,\,\,\,\\
e_n^{(3)} &\eqdef&	\|\hat M_{n,3}  - M_{3}\| \leq A^{9/2}\sqrt{\sum_{i=1}^n\gamma_i^{-2}\frac{\log(4A^3/\delta)}{2n^2}}\,.
\eeqan
\end{mycorollary}

\begin{algorithm}[!hbtp]
	\caption{{\bf EstimateFeatures}} \label{alg:MainIntuitive}
	\begin{algorithmic}[1]
          \STATE {\bf Input:} \#sessions $n$; \#mini-sessions $\ell$; (user, action, reward) tuples $(b_i, a_{i,l}, X_{a_{i,l},b_i,i,l})_{1 \leq i \leq n, 1 \leq l \leq \ell}$.
		\STATE Compute the 
$A\times A$	matrix $\hat M_{n,2} = (\tilde r_{a,a',n})_{a,a'\in[A]}$
and the $A\times A\times A$ tensor $\hat M_{n,3} = (\hat r_{a,a',a'',})_{a,a',a''\in[A]}$.
	\STATE Compute a $A\times C$ whitening matrix $\hat W_n$ of $\hat M_{n,2}$ \\
	\COMMENT{Take $\hat W_n = \hat U_n \hat D_n^{-1/2}$ where $\hat D_n$ is the $C\times C$ diagonal matrix with the top $C$ eigenvalues of  $\hat M_{n,2}$, and $\hat U_n$ the $A\times C$ matrix of corresponding eigenvectors.} 
	\STATE Form the  $C\times C \times C$ tensor $\hat T_n = \hat M_{n,3}(\hat W_n,\hat W_n,\hat W_n)$.
	\STATE Apply the RTP algorithm \citep{Anandkumar} to $\hat T_n$, and compute its robust eigenvalues $(\hat \lambda_{n,c})_{c\in[C]}$ with eigenvectors $(\hat \phi_{n,c})_{c\in [C]}$. \\
        \COMMENT{The paper of \citet[Sec. 4]{Anandkumar} defines eigenvalues/eigenvectors of tensors.}
	\STATE Compute for each $c\in[C]$, $\ol u_{n,c} = \lambda_{n,c}(\hat W_n^\top)^\dagger \hat \phi_{n,c}$ and $\ol v_{n,c} = \lambda_{n,c}^{-2}$.
	\STATE {\bf Output: Estimate of latent classes $U$}: The $A\times C$ matrix $\overline{U}_n$ obtained by stacking the vectors $\ol {\bf u}_{n,c}\in\Real^A$ side by side.
	\end{algorithmic} 
\end{algorithm}

{\bf Reconstruction algorithm.}  The EstimateFeatures algorithm
(Algorithm~\ref{alg:MainIntuitive}) employs a whitening matrix
$\hat W_n$, of the empirical estimate of the matrix $M_2$, to build
the empirical tensor $\hat T_n$.  This tensor is then used to recover
the columns of the matrix $U=(u_{a,c})_{a\in[A],c\in[C]}$ via the RTP algorithm.  For the sake of
completeness, we also introduce $W$, a whitening matrix of $M_{2}$
(i.e., $W^T M_2 W = I$), the corresponding tensor $T = M_{3}(W,W,W)$,
and finally the estimation error $e_n \eqdef \|\hat T_n - T \|$.

{\bf Reconstruction guarantee.}  
Our next result makes use of the following proposition from
\citet[Theorem 5.1]{Anandkumar}, restated here for completeness.

\begin{myproposition}[Theorem 5.1 of \citet{Anandkumar}]\label{prop:RTP}
  Let $\hat T = T+E\in \Real^{C\times C\times C}$, where $T$ is a
  symmetric tensor with orthogonal decomposition
  $T = \sum_{c=1}^C \lambda_c \phi_c^{\otimes 3}$, where each
  $\lambda_c>0$, $\{\phi_c\}_{c\in[C]}$ is an orthonormal basis, and
  $E$ is a symmetric tensor with operator norm $||E||\leq \epsilon$.
  Let $\lambda_{\min} = \min\{\lambda_c:c\in[C]\}$,
  $\lambda_{\max} = \max\{\lambda_c:c\in[C]\}$. Run the RTP algorithm
  with input $\hat T$ for $C$ iterations. Let
  $\{(\hat \lambda_c,\hat \phi_c)\}_{c\in[C]}$ be the corresponding
  sequence of estimated eigenvalue/eigenvector pairs returned.  Then,
  there exist universal constants $C_1,C_2>0$ for which the following
  is true. Fix $\eta\in(0,1)$ and run RTP with parameters (i.e.,
  number of iterations) $L,N$ with $L = poly(C)\log(1/\eta)$, and
  $N\geq C_2\Big(\log(C)
  +\log\log\big(\frac{\lambda_{\max}}{\epsilon}\big)\Big)$.
  If $\epsilon \leq C_1 \frac{\lambda_{\min}}{C}$, then with
  probability at least $1-\eta$, there exists a permutation
  $\pi \in \mathbb{S}_C$ such that
  \begin{align*}
    &\forall c\in[C]: \quad |\lambda_c - \hat \lambda_{\pi(c)}| \leq 5 \epsilon, \quad \; ||\phi_c - \hat \phi_{\pi(c)}|| \leq 8\epsilon/\lambda_c, \\
    &\text{ and } \quad ||T - \sum_{c=1}^C \hat\lambda_c{\hat\phi_c}^{\otimes 3}|| \leq 55\epsilon\,.
  \end{align*}
\end{myproposition}

Lemma~\ref{lem:Mestimates} gives a decomposition of the (symmetric)
tensor $M_3$, but it may be not orthogonal;  standard
transformation \citep[Sec. 4.3]{Anandkumar} gives an orthogonal
decomposition for the tensor\footnote{For a 3rd order tensor
  $A \in \mathbb{R}^{a \times a \times a}$ and 2nd order tensor or
  matrix $B \in \mathbb{R}^{a \times b}$,
  $A(B,B,B) \in \mathbb{R}^{b \times b \times b}$ is the 3rd order
  tensor defined by
  $[A(B,B,B)]_{i_1,i_2,i_3} \eqdef \sum_{j_1,j_2,j_3 \in
    [n]}A_{j_1,j_2,j_3} B_{j_1,i_1} B_{j_2,i_2} B_{j_3,i_3}$.
  See \citet{Anandkumar} for more details on
  notation and results.}  $M_3(W,W,W)$, with $W$ a
matrix that whitens $M_2$. We can thus use Proposition~\ref{prop:RTP}
with $T= M_3(W,W,W)$, $\hat T = \hat T_n$, $\epsilon = e_n$ and
$\eta=\delta$ in order to prove the following guarantee (Theorem~\ref{thm:Uestimates}) on the
recovery error between columns of $U$ and their estimate. 

We now introduce mild separability conditions on the mixture
weights ${\bf v}_b$ and the spectrum of the 2nd moment matrix $M_2$
needed for the reconstruction guarantee to hold, similar to
those assumed for \citet[Theorem~2]{lazaric2013sequential}.

\begin{assumption}
  \label{ass:1}
  There exist positive constants
  $v_{\min}$, $\sigma_{\min}$, $\sigma_{\max}$ and $\Gamma$
  such that

  \vspace{-6mm}
        \begin{align*}
          &\min_{b\in[B],c\in[C]}v_{b,c} \geq v_{\min},\\
          &\forall c\in[C], \; \sigma_c = \sqrt{\lambda_c(M_2)} \in[\sigma_{\min},\sigma_{\max}] \quad \text{  and } \\
          &\min_{c\neq c' \in [C]\times[C]} |\sigma_c-\sigma_{c'}| \geq \Gamma\,,          
        \end{align*}
        
\noindent
	where $\lambda_c(A)$ denotes the $c^{th}$ top eigenvalue of $A$.
\end{assumption}

 \begin{mytheorem}[Recovery guarantee for online estimation of user classes $U$]
\label{thm:Uestimates}
Let Assumption \ref{ass:1} hold, and let $\delta\in(0,1)$. If the
number of mini-session satisfies

\vspace{-6mm}
\begin{align*}
  &\frac{n^2}{\sum_{i=1}^n\gamma_i^{-2}} \geq
\max  \Bigg\{
\frac{2A^6\log(4A^2/\delta)}{\min\{\Gamma,\sigma_{\min}\}^2}, 
\frac{A^9(1 + 10(\frac{1}{\Gamma} +
  \frac{1}{\sigma_{\min}})(1+u_{\max}^3))^2C^{5}\log(4A^3/\delta)}{2C^2_1\sigma_{min}^{3}}
\Bigg\},
\end{align*}
then with probability at least $1-2\delta$, there exists some
permutation $\pi \in \mathbb{S}_C$ such that for all $c\in[C]$, the
output $\bar U_n$ of the {\em EstimateFeatures} algorithm satisfies
\beq ||{\bf u}_c - \ol {\bf u}_{n,\pi(c)}|| \leq \Diamond
A^{3}\sqrt{\sum_{i=1}^n\gamma_i^{-2}\frac{C\log(4A^3/\delta)}{2n^2}}.\label{eqn:Uestimates}
\eeq 
where ${\bf u}_c = (u_{a,c})_{a\in[A]}$.
Here, the constant (we use the "diamond" symbol to denote it) is
\begin{align*}
&\Diamond = \Big(\frac{CA}{\sigma_{min}}\Big)^{3/2}
	\Bigg(	13\sqrt{\sigma_{max}} + 4
	 \sqrt{2\min\{\Gamma,\sigma_{\min}\}}
\end{align*}

\vspace{-8mm}
\begin{align*}
	 & +5\left(\frac{\sigma_{max}}{\Gamma}+ 
	 \frac{1}{2\sigma_{max}}\right)
	  \min\{\Gamma,\sigma_{\min}\}
	  \Bigg)
	 \aleph\\
    	  &+\,\, \left(\frac{2\sigma_{max}}{\Gamma}+ \frac{1}{\sigma_{max}}\right)\frac{1}{v_{\min}^2} \\
 &+ 5\sqrt{3/8}\Big(\sqrt{\sigma_{\max}} + 
	  	 \sqrt{ \min\{\Gamma,\sigma_{\min}\}/2}\Big)
	  	 \left(\frac{2CA}{\sigma_{\min}}\right)^3 
  \aleph^2
	  	 \min\{\Gamma,\sigma_{\min}\}\,,
\end{align*}

\noindent
with the notation $\aleph = 1+10(\frac{1}{\Gamma} + \frac{1}{\sigma_{\min}})(1+u_{\max}^3)$.
\end{mytheorem}

The proof strategy follows that of \citet[Theorem
2]{lazaric2013sequential} and is detailed in the appendix for clarity. It consists in relating, on the one hand, the
estimation errors $e_n^{(2)}$ of $M_2$ and $e_n^{(3)}$ of $M_3$ from
Corollary~\ref{cor:estimates} to the condition
$\epsilon \leq C_1 \frac{\lambda_{\min}}{C}$, and, on the other hand,
relating the reconstruction error on the columns of $U$ to the control
on the terms $|\lambda_c - \hat \lambda_{\pi(c)}|$ and
$||\phi_c - \hat \phi_{\pi(c)}||$ coming from
Proposition~\ref{prop:RTP}.  We note that the bound appearing in the
condition on the number of mini-sessions is potentially large (due to
the terms $A^6, C^5$, etc.). This is due to the combination of the RTP
method with the importance sampling scheme, and it remains unclear if
the bound can be significantly improved within this framework.
\section{{Recovering latent mixture distributions ($v_b$):
    robustness of the OFUL algorithm}}\label{sec:featuresV}

In order to recover the weights vectors ${\bf v}_{b}\in\Real^C$ and
thus the matrix $V$, it would be tempting to use again an instance of
the RTP method but this time to aggregate across actions, i.e., by
forming a $B\times B$ and $B\times B\times B$ tensor.  Unfortunately,
aggregation of elements of $U$ fails for two reasons: First, we do not
have different views across users $b$, contrary to what we have for
actions $a$. It is thus hopeless to be able to form an estimate of the
2nd and 3rd moment tensors as before. Second, and rather technically,
convex combinations of the $\{u_{a,c}\}_{a\in[A]}$ need not be
positive. This prevents the application of the RTP method which
requires positive weights to work.

We thus consider a different strategy that uses an algorithm designed
for linear bandits. However since the feature matrix $U$ is unknown a
priori and can only be estimated, we need to work with perturbed
features. A first solution is to propagate the additional error
resulting from the error on the features in the standard proof of
\OFUL.  However, this leads to a sub-optimal regret that is no longer
scaling as $\tilde O(\sqrt{T})$ with the time horizon. We overcome
this hurdle by showing in Theorem~\ref{thm:pOFUL2} a robustness
property of \OFUL\, of independent interest, which aids us in
controlling the regret of the overall latent class algorithm
(Algorithm~\ref{alg:PerOFULwithExploration}).
Consider \OFUL\ run with perturbed (not necessarily linearly
realizable) rewards. Formally, consider a finite action set
$\mathcal{A} = \{1, 2, \ldots, A\}$ and distinct feature vectors
$\{\bar{\bf u}_a \in \mathbb{R}^{C \times 1}\}_{a \in \mathcal{A}}$.
Let
$\bar{U}^\top \bydef [\bar{\bf u}_1 \; \bar{\bf u}_2 \; \ldots \;
\bar{\bf u}_A] \in \mathbb{R}^{C \times A}$.
The expected reward when playing action $A_t = a$ at time $t$ is
denoted by $m_a \bydef \expect{Y_t \given A_t = a}$, with
${\bf m} \bydef \left( m_a \right)_{a \in \mathcal{A}}$. Let us assume
that there exists a unique optimal action for the expected rewards
${\bf m}$, i.e., $\arg \max_{a \in \mathcal{A}} m_a = \{a^\star \}$,
with the regret at time $n$ being
$\kR_n \bydef \sum_{t=1}^n \left( m_{a^\star} - m_{A_t} \right)$.  The
key point here is that ${\bf m}$ need not be linearly realizable
w.r.t. the actions' features -- we will not require that
$\min_{{\bf v} \in \mathcal{R}^C} \norm{{\bf m} - \bar{U} {\bf v}}$ be
$0$.
\begin{algorithm}[!hbtp]
  \caption{\OFUL\ (Optimism in Face of Uncertainty for Linear
    bandits) \citep{AbbPalCsa11:linbandits}} \label{alg:OFUL}
	\begin{algorithmic}
          \REQUIRE{Arms' features $\bar U$, regularization parameter $\lambda$, norm parameter $R_\Theta$}
          \FOR{all times $t \geq 1$}
          \STATE {\bf 1.} Form the $C \times (t-1)$ matrix
          $\bar{\mathbf{U}}_{1:t-1} \bydef [\bar{\bf u}_{A_1} \;
          \bar{\bf u}_{A_2} \ldots \; \bar{\bf u}_{A_{t-1}}]$
          consisting of all arm features played up to time $t-1$, and
          $\mathbf{Y}_{1:t-1} \bydef (Y_1, \ldots, Y_{t-1})^\top$. Set
          $V_{t-1} \bydef \lambda I + \sum_{s=1}^{t-1} \bar{\bf u}_{A_s}
          \bar{\bf u}_{A_s}^\top$.  
          \STATE {\bf 2.} Choose the action          
          
          \vspace{-8mm}
          \begin{align*}
            &A_t \in \arg \max_{a \in \mathcal{A}} \max_{{\bf v} \in \cC_{t-1}}
              \bar{\bf u}_a^\top {\bf v}, \quad \text{where} \\
            &\cC_{t-1} \bydef \{{\bf v} \in \mathbb{R}^C: \norm{{\bf v} -
                        \hat{{\bf v}}_{t-1}}_{V_{t-1}} \leq D_{t-1} \}, \\
            &D_{t-1} \bydef R \sqrt{2 \log \left( \frac{\det(V_{t-1})^{1/2}
                      \lambda^{-C/2}}{\delta} \right)}\!+\! \lambda^{1/2} R_\Theta \\
            &\hat{{\bf v}}_{t-1} \bydef V_{t-1}^{-1}\bar{\mathbf{U}}_{1:t-1}
                                  \mathbf{Y}_{1:t-1.}.
          \end{align*}
          \vspace{-8mm}
          \ENDFOR
	\end{algorithmic} 
\end{algorithm}

{\bf \OFUL\ Regret with linearly realizable rewards.} The \OFUL\
algorithm is stated for the sake of clarity as Algorithm
\ref{alg:OFUL}.  Before studying the linearly non-realizable case, we
record the well-known regret bound for it in the unperturbed
case, that is when
$\forall a\in[A], m_a = \bar{\bf u}_{a}^\top{\bf v}^\star$ for some
unknown ${\bf v}^\star$.

\begin{mytheorem}[{\OFUL} regret \citep{AbbPalCsa11:linbandits}]
  Assume that $||{\bf v}^\star||_2\leq R_\Theta$, and that for all
  $a\in\cA$, $||\bar{\bf u}_a||_2 \leq R_\cX$,
  $|\langle \bar{\bf u}_a,{\bf v}^\star \rangle| \leq 1$.
Then
with probability at least $1-\delta$, the regret of 
\OFUL\ satisfies: $\forall n\geq 0$,

\vspace{-8mm}
\begin{align*}
   \kR_n &\leq 4  \sqrt{nC\log(1+nR_\cX^2/(\lambda C)} \quad \times \\ 
  &\Big(\lambda^{1/2}R_\Theta + R\sqrt{2\log(1/\delta)+C\log(1+nR_\cX^2/(\lambda C)})\,,
\end{align*}

\noindent
provided that the regularization parameter $\lambda$ is chosen such
that $\lambda \geq \max\left\{1,R_\cX^2, 1/R_\Theta^2\right\}$.
\end{mytheorem}

{\bf  Regret of \OFUL\ with Perturbed Features.} 
We make a structural definition to present the result. Let
$\alpha(\bar{U}) \bydef \max_{J} \norm{\mathbf{A}_J^{-1}}_2$, where
$\mathbf{A} = \left[ \begin{array}{c}
                       \bar{U} \\
                       I_{C} \end{array} \right] \in \mathbb{R}^{(A+C)
  \times C}$,
  $\mathbf{A}_J$ is the $C \times C$ submatrix of $\mathbf{A}$ formed
  by picking rows $J$, and $J$ ranges over all size-$C$ subsets of
  full-rank rows of $\mathbf{A}$.  We will require for our purposes
  that $\alpha(\bar{U}^\mathsf{T})$ is not too large. For intuition
  regarding $\alpha$, we refer to \citet{forsgren1996linear} (the
  final 3 paragraphs of p. 770, Corollary 5.4 and section 7). We
  remark that the condition that $\alpha(\bar{U}^\mathsf{T})$ be small
  is analogous to a $\gamma$-incoherence type property commonly used
  in prior work \citep[Assumption A2]{BreCheSha14:latentCF}, stating
  that two distinct feature vectors ${\bf u}_c$ and ${\bf u}_{c'}$,
  $c \neq c'$, must have a minimum angle separation.

  Let ${\bf v}^\circ \in \mathbb{R}^C$ be arbitrary with $\ell^2$ norm
  at most $R_\Theta$ (it helps to think of $\bar{U} {\bf v}^\circ$ as
  an approximation of ${\bf m}$),
  $\epsilon_a \bydef m_a - \bar{\bf u}_a^\top {\bf v}^\circ$,
  $\epsilon \bydef (\epsilon_a)_{a \in \mathcal{A}} \in \mathbb{R}^A$.
  We now state a robustness result for \OFUL\ potentially of
  independent interest.

\begin{mytheorem}[\OFUL\ robustness property]
  \label{thm:pOFUL2}
  Suppose $||{\bf v}^\circ||_2\!\leq\! R_\Theta$, $\lambda\! \geq\! \max\! \left\{\!1,R_\cX^2,1\!/{4 R_{\Theta}^2} \right\}$, $\forall a\!\in\!\cA$,
  $||\bar{\bf u}_a||_2 \leq R_\cX$ and $|m_a| \leq 1$.
  If the deviation from linearity satisfies
  
  \vspace{-3mm}
   \beq \norm{\epsilon}_2
  \equiv \norm{{\bf m} - \bar{U} {\bf v}^\circ}_2 < \min_{a \neq
    a^\star} \; \frac{\bar{\bf u}_{a^\star}^\top {\bf v}^\circ -
    \bar{\bf u}_{a}^\top {\bf v}^\circ}{2\alpha(\bar{U}^\top)
    \norm{\bar{\bf u}_{a^\star} - \bar{\bf
        u}_{a}}_2}, \label{eqn:ass2} \eeq 
        
\noindent
        then, with probability at
  least $1 - \delta$ for all $T \geq 0$,
  
  \vspace{-2mm}
  \begin{align*}
    \kR_T &\leq 8 \rho' \sqrt{TC \log\left(1 +
      \frac{TR_\cX^2}{\lambda C}\right)} 
    \left(\lambda^{1/2}R_\Theta +
    R\sqrt{2 \log\frac{1}{\delta} + C \log
      \left(1+\frac{TR_\cX^2}{\lambda C}\right)} \right), 
  \end{align*}
  
\noindent
 where
    $\rho' \bydef \max\left\{1,  \max_{a \neq a^\star} \; \frac{m_{a^\star} -
      m_a}{\bar{\bf u}_{a^\star}^\top {\bf v}^\circ - \bar{\bf u}_a^\top {\bf v}^\circ } \right\}$.
 
\end{mytheorem}

Theorem~\ref{thm:pOFUL2} essentially states that when the deviation of
the actual mean reward vector from the subspace spanned by the feature
vectors is small, the \OFUL\ algorithm continues to enjoy a favorable
$O(\sqrt{T})$ regret up to a factor $\rho' \geq 1$. The quantity
$\rho'$ in the result is a geometric measure of the distortion in the
arms' actual rewards ${\bf m}$ with respect to the (linear)
approximation $\bar{U} {\bf v}^\circ$. We control this quantity in the
next paragraph. (Note that $\rho' = 1$ in the perfectly linearly
realizable case $\epsilon = 0$, and this gives back the standard
\OFUL\ regret up to a universal multiplicative constant.)


{\bf Applying the Robust analysis of \OFUL\ to the Low-rank Bandit setup.}
In this paragraph, we translate Theorem \ref{thm:pOFUL2} to our Low Rank
Bandit (LRB) setting in which OFUL uses feature vectors with noisy
perturbations (estimated by, say, a Robust Tensor Power (RTP)
algorithm). Throughout this section, we fix a user
$b$.

\begin{table}[htbp]
  \begin{center}
    \begin{tabular}{c|c}
      \OFUL & LRB \\
      \hline
      ${\bf v}^\circ$ & ${\bf v}_b \equiv \left(v_{b,c} \right)_{c \in [C]} \in \mathbb{R}^C$ \\
      $\bar{U}$ & $\bar{U}_n \in \mathbb{R}^{A \times C}$ \\ 
      ${\bf m}$ & ${\bf m}_b\equiv U {\bf v}_b \in \mathbb{R}^A$ \\
          $a^\star$ & $a_b^\star \bydef \arg \max_{a \in [A]} {\bf u}_a^T {\bf v}_b$ \\
          $\epsilon_a$ & $({\bf u}_a - \bar{{\bf u}}_{n,a})^\top {\bf v}_b$ \\
          $\epsilon \equiv (\epsilon_a)_{a \in \mathcal{A}}$ & $(U - \bar{U}_{n}) {\bf v}_b$   
        \end{tabular}
  \end{center}
  \caption{Correspondences between \OFUL\ and Low Rank Bandit (LRB) quantities at time $n$ and for user $b$}
  \label{tab:pOFUL-LRB}
\end{table}

We can now translate Theorem \ref{thm:pOFUL2}  thanks to the correspondence with the perturbed OFUL setting:  In our low-rank bandit setting, the matrix $\bar{U}=\bar{U}_n$ depends on the reconstruction algorithm 
at mini-session $n$. Moreover, the optimal action  $a^\star \equiv a^\star_b$ now depends on the user $b$. We denote for a user $b\in [B]$ the minimum gap across suboptimal actions to be $g_b\eqdef\min_{a\neq a^\star_b} ({\bf u}_{a_b^\star}-{\bf u}_a)^\top {\bf v}_b$.
Likewise, the error vector $\epsilon$ depends on $b,n$.
Its norm $||\epsilon||_2$ appears in the condition \eqref{eqn:ass2} 
and the definition of $\rho$, and is controlled by the reconstruction error of Theorem~\ref{thm:Uestimates}. It decays with the number of mini-sessions $n$.

We define $\alpha_n \eqdef \alpha(\bar{U}_n)$, $\alpha_\star \eqdef \alpha(U)$
and use $\max_{b}||{\bf v}_b||$ for $R_\Theta$,
%
Using these notations, and adapting the proof of Theorem~\ref{thm:pOFUL2} to handle a variable $\ol U_n$, we can now translate 
the result of the perturbed \OFUL\ to our LRB setting:

\begin{mylemma}\label{lem:pOFULSetting} 
  Let $0 < \delta \leq 1$ and $b \in [B]$. Provided that the number
  of mini-sessions $n_0$ satisfies
  $\frac{n_0^2}{\sum_{i=1}^{n_0}\gamma_i^{-2}} \geq
  \varhexagon_{b,\delta}$, where we introduced the notation
  
  \vspace{-9mm}
\begin{align*}
\varhexagon_{b,\delta} =   &\max\Bigg\{ \frac{2A^6\log(4A^2/\delta)}{\min\{\Gamma,\sigma_{\min}\}^2}, \\
%
                     &\frac{A^9(1 + 10(\frac{1}{\Gamma} + \frac{1}{\sigma_{\min}})(1+u_{\max}^3))^2C^{5}\log(4A^3/\delta)}{2C^2_1\sigma_{min}^{3}} \quad \times \\
                     &
                        \Diamond^2 A^6C^2\log(4A^3/\delta) \quad \times \\
  &\max\bigg\{
                        2\alpha_\star^2,
                        \frac{8A||{\bf v}_b||_2^2}{g_b^2},
                        \frac{2^7\alpha_\star^2 Cu_{\max}^2 ||{\bf v}_b||_2^2}{g_b^2}
                        +\frac{1}{2}
                        \bigg\}
                        \Bigg\},
\end{align*}

\vspace{-3mm}
  then with probability at least $1 - 2\delta$,
  $||\epsilon||_2= ||(U - \bar{U}_{n}) {\bf v}_b||_2$ is small enough that for any $n \geq n_0$,
  condition~\eqref{eqn:ass2} is satisfied.
  Consequently, Theorem~\ref{thm:pOFUL2} applies with 
  
  \vspace{-6mm}
  \begin{align*}
    &R_\Theta=\max_{b}||{\bf v}_b||_2, \quad R_\cX =
  \max_{a\in\cA}||{\bf u}_a||_2 + \frac{\sqrt{A}}{2\alpha_\star},  \quad\text{and} \\
    &\rho'\equiv
  \rho'_{n,b} \leq 2.
  \end{align*}
\end{mylemma}

Thus, provided that the total number of mini-sessions of interaction (not necessarily corresponding to interactions with user $b$) is large enough,
then the \OFUL\ algorithm run during interactions with user $b$ will achieve a controlled regret. However, we want to warn that the $\varhexagon_{b,\delta}$ resulting from the RTP method, especially the second term of the max, 
may be potentially large, although being a constant.

\section{{Putting it together: Online Recommendation
    algorithm}}\label{sec:Recommend}

This section details our main contributions for recommendations in the
context of mini-sessions of interactions with unknown mixtures of
latent profiles: first Algorithm~\ref{alg:PerOFULwithExploration} that
combines RTP with \OFUL, and then a regret analysis in
Theorem~\ref{thm:main}.

The recommendation algorithm we propose
(Algorithm~\ref{alg:PerOFULwithExploration}) uses the RTP method to
estimate the matrix $U$ and then applies \OFUL\ to determine an
optimistic action.  Importantly, it finally outputs a distribution
that mixes the optimistic action with a uniform exploration.  The
mixture coefficient goes to $0$ with the number of rounds, thus
converging to playing \OFUL\ only. It ensures that the importance
sampling weights are bounded away from $0$ in the beginning.

{\bf Main analytical result: Regret bound}
%

\begin{mytheorem}[Regret of Algorithm~\ref{alg:PerOFULwithExploration}]
\label{thm:main}
With Assumption \ref{ass:1} holding, let $\delta\in(0,1)$,
$\varhexagon_\delta = \max_{b\in[B]} \varhexagon_{b,\delta}$ (from
Lemma~\ref{lem:pOFULSetting}), and let $n_0$ be the first mini-session
at which
    $\frac{n_0^2}{\sum_{i=1}^{n_0}\gamma_i^{-2}} \geq \varhexagon_{\delta}$.
%
    The regret of Algorithm~\ref{alg:PerOFULwithExploration} at time
    $T=N\ell$ (acting for $N$ mini-sessions of length $\ell$) using
    internal instances of \OFUL\ parameterized by $\delta > 0$
    satisfies
    
    \vspace{-6mm}
\begin{align*}
\Esp[\kR_T]  \leq   &16
    \sqrt{BT C \log\left(1 +
    \frac{TR_\cX^2}{\lambda C}\right)} 
  \left(\lambda^{1/2}R_\Theta + R\sqrt{2 \log\frac{1}{\delta} + C
    \log
    \left(1+\frac{T R_\cX^2}{\lambda C}\right)} \right) \\
  & +\ell(n_0-1 + \sum_{n=n_0}^N \gamma_n) + 3\delta T\,,
\end{align*}

\noindent
provided that $\lambda \geq \min\{1,R^2_\cX,1/R^2_\Theta\}$, with
$R_\Theta\geq \max_{b}||{\bf v}_b||_2$,
$R_\cX \geq \max_{a\in\cA}||{\bf u}_a||_2 +
\frac{\sqrt{A}}{2\alpha_\star}$.
Consequently, choosing $\delta = 1/T$ and
$\gamma_n = \sqrt{\log(n+1)/n}$, $n \in \Nat$, say, yields the order 
$\Esp[\kR_T] = O\left( C \sqrt{BT} \log T \right).$

\end{mytheorem}

{\bf Discussion.}  (1) The regret of
Algorithm~\ref{alg:PerOFULwithExploration} scales with $T$ similar to
that of an \OFUL\ algorithm run with perfect knowledge of the feature
matrix $U$: $\tilde O(C\sqrt{BT})$.  This is a non-trivial result as
$U$ is not assumed to be known a priori and is estimated by
Algorithm~\ref{alg:PerOFULwithExploration} using tensor methods.

\begin{algorithm}[!htbp]
	\caption{Per-user \OFUL\ with exploration} \label{alg:PerOFULwithExploration}
	\begin{algorithmic}[1]	
          \REQUIRE Parameters $\lambda$, $R_\Theta$ for \OFUL,
          exploration rate parameters $\gamma_n$, $n \geq 1$.
		\FOR{mini-session $n=1,\dots,N$}
		\STATE Get user $b_n$.
		\STATE Let $p_n \sim \mbox{Bernoulli}(\gamma_n)$
		\IF{$p_n=0$}
                \STATE \COMMENT{Carry out an {\bf ESTIMATE} mini-session}
		\FOR{step $k=1,2,\dots,\ell$}	
		\STATE Output $a_{n,k} \sim \text{Uniform}([A])$.
		\ENDFOR

                \STATE Let $\ol{U}_{n} = \textbf{EstimateFeatures}$
                (Algorithm \ref{alg:MainIntuitive}) with input
                $(b_i, a_{i,l}, X_{a_{i,l},b_i,i,l})_{
                1\!\leq\! i\leq\! n, 1\! \leq\! l \leq\! \ell, p_n\! =\! 0}$
                \\\COMMENT{Update feature estimates using samples from
                  previous {\bf ESTIMATE} mini-sessions}

		\ELSE
                \STATE \COMMENT{Carry out an {\bf \OFUL\ }mini-session}
                \FOR{step $k=1,2,\dots,\ell$} 

                \STATE Run one iteration of \OFUL\ (Algorithm
                \ref{alg:OFUL}) with features $\ol{U}_n$, parameters
                $\lambda$ and $R_\Theta$, and past actions and rewards
                $(a_{i,l}, X_{a_{i,l},b_i,i,l})$,
                $1 \leq i < n, 1 \leq l \leq \ell$, for which
                $p_i = 1$ and $b_i = b_n$

                \COMMENT{An instance of \OFUL\ for each user using
                  current feature estimates, and observed actions and
                  rewards from previous {\bf \OFUL\ }mini-sessions}

                \STATE Output action $a_{n,k}$ returned by \OFUL\
		\ENDFOR		
		\ENDIF
		\ENDFOR		
	\end{algorithmic} 
\end{algorithm}

(2)
One can also compare the result with the regret of ignoring the
mixture (low-rank) structure and simply running an instance of UCB per
user, which would scale as $O(\sqrt{ABT})$. This becomes highly
suboptimal when the number of actions/items $A$ is much larger than
the number of user types $C$, demonstrating the gain from leveraging
the mixed linear structure of the problem. Note also that we do not
need a specific user to interact for a long time but for as few as
$\ell\geq3$ consecutive steps, contrary for instance to the transfer
method \citep{lazaric2013sequential}, where a large number of
consecutive interaction steps with the same user is required.

(3) It
is worthwhile to contrast the result and approach with that in
\citet{DjoKraCev13:GPbandits} -- the authors there incur
an additional regret term due to the error in approximately estimating
the low-rank matrix, which requires additional tuning ending up with a
regret of $O(T^{4/5})$. On the other hand, we avoid this
approximation error by showing and exploiting the robustness property
of OFUL, which guarantees $\sqrt{T}$ regret as soon as the estimated
features $\tilde{U}$ are within a small radius of the actual ones.

The result (and analysis) does come with a caveat that the
model-dependent term $\varhexagon_\delta$, although being independent on
the time horizon $T$, is potentially large. With $\gamma_n$ set as
in Theorem~\ref{thm:main}, it appears as an additive exponential constant term in
the regret\footnote{With additional prior knowledge of $\gamma_n$, the
  dependence of the additive term can be made polynomial in
  $\varhexagon_\delta$: choosing
  $\gamma_n = \min\{1,\sqrt{\varhexagon_\delta/n}\}$, it holds that
  $\ell(n_0-1 + \sum_{n=n_0}^N \gamma_n) \leq 2\sqrt{\varhexagon_\delta
    \ell T}+ \ell\,.$}.
This arises 
from the RTP method, and it is currently
unclear if this term can be significantly reduced with the current
line of analysis.
Numerical evidence, however, indicates that no such large additive
constant enters into the regret (Section \ref{sec:Recommend}). Also,
on the bright side, note that $\varhexagon_\delta$ does not need to be
known by the algorithm.

{\bf Numerical Results.} The performance of the low-rank bandit
strategy (Algorithm \ref{alg:PerOFULwithExploration}) is shown in
Figure \ref{fig:plot}, simulated for $20$ users arriving uniformly at
random, $3$ user classes and $200$ actions. Both the latent class
matrix $U_{200 \times 3}$ the mixture matrix $V_{20 \times 3}$ are
random one-shot instantiations. The proposed algorithm (Algorithm \ref{alg:PerOFULwithExploration}), with two different
exploration rate schedules $\tilde{O}(n^{-1/2})$ and
$\tilde{O}(n^{-1/3})$ ('RTP+OFUL(sqrt)' and 'RTP+OFUL(cuberoot)' in the figure), is compared with (a) basic UCB ('UCB' in the figure) ignoring the
linear structure of the problem (i.e., UCB per-user with $200$
actions), (b) \OFUL\ per-user with complete knowledge of the user
classes and $p_n = 1$ always, i.e., no exploration mini-sessions, and (c) An implementation of the Alternating Least Squares estimator \citep{takacs2012alternating,mary2014bandits} for the matrix $U$ along with \OFUL\ per-user. The
proposed algorithm, with the theoretically suggested exploration
$\tilde{O}(n^{-1/2})$, is observed to exploit the latent structure
considerably better than simple UCB, and is not too far from the
unrealistic \OFUL\ strategy which enjoys the luxury of latent class
information. It is also competitive with performing Alternating Least Squares, which does not come with analytically sound performance guarantees in the bandit learning setting. Also, the large additive constants in the theoretical
bounds for Algorithm \ref{alg:PerOFULwithExploration} do not manifest here.

%

\begin{figure}[ht]
\begin{center}
\centerline{\includegraphics[width=0.92\columnwidth]{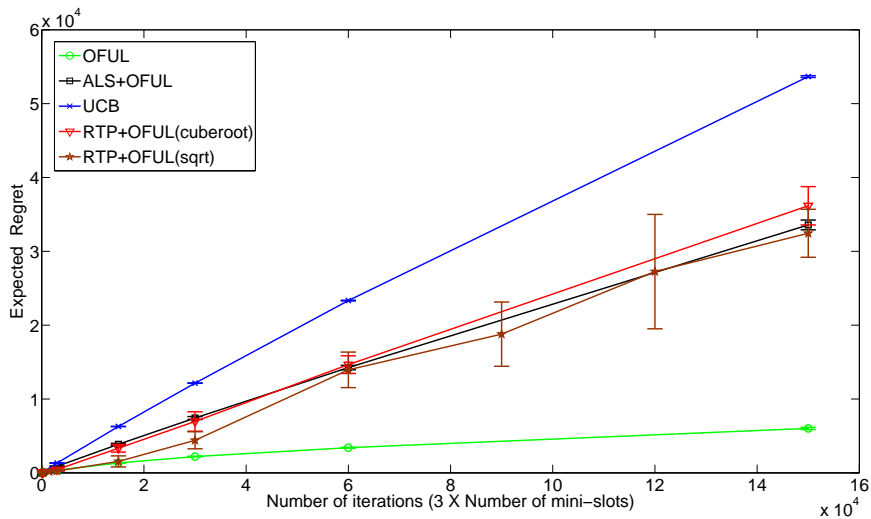}}

\vspace{-3mm}
\caption{Regret of the proposed algorithm (`RTP+OFUL' or Algorithm
  \ref{alg:PerOFULwithExploration}) for two different exploration rate
  schedules, compared with (a) independent UCB per-user, (b)
  \OFUL\ per-user with perfect knowledge of latent classes $U$, and (c) Alternating Least Squares estimation for the matrix $U$, along with \OFUL\ per-user. Here,
  $B = 20$ users, $C = 3$ classes, and $A = 200$, with
  randomly generated $U$ and $V$. Plots show the sample mean
  of cumulative regret with time, with $1$ standard deviation-error
  bars over $10$ sample experiments.}
\label{fig:plot}
\end{center}
\vskip -0.2in
\end{figure} 

{\bf Related work.}
\label{sec:relatedwork}
The popular low-rank matrix completion problem studies the recovery
$U$ and $V$ given a small number of entries sampled at random from
$UV^T$ with both $U$ and $V$ being tall matrices, see for instance
\citet{jain2013low} and citations therein. However, its setting is
different than ours for several reasons. It typically deals with batch
data arising from a sampling process that is not active but uniform
across entries of $UV^T$. Further, it requires sensing operators
having strong properties (such as the RIP property), and most
importantly, the performance metric is not regret but reconstruction
error (Frobenius or $2$-norm). 

In the linear bandit literature
\citep{AbbPalCsa11:linbandits,rusmevichientong2010linearly,Dani2008Lin}, the key
constraining assumption is that either user side ($V$) or item side
($U$) features are precisely and completely known a priori. In
contrast, the problem of low regret recommendation across users with
latent mixtures does not afford us the luxury of knowing either $U$ or
$V$, and so they must be learnt ``on the fly''. Another related work
in the context of bandit type schemes for latent mixture model
recommender systems is that of \citet{BreCheSha14:latentCF}, in which,
under the very specific uniform mixture model for all users, they
exhibit strategies with good regret.

\citet{nguyen2014cold} consider an alternating
minimization type scheme in linear bandit models with two-sided
uncertainty (an alternative model involving latent
``factors''). However no rigorous guarantees are given for the bandit
schemes they present; moreover, it is not known if alternating
minimization finds global minima in general. Another related work is
in the transfer learning setting from \citet{lazaric2013sequential}:
The method combines the RTP method
\citep{Anandkumar,anandkumar2012method} essentially with a
standard UCB \citep{AuerEtAl02FiniteTime}, but however works in the
setting of a large number interactions with a same user, without
assuming access to ``user ids''.  As a result, the regret bound in
this setting scales linearly with the number of rounds. Our result in
this paper shows that with additional access to just user identifiers,
we can reduce the regret rate to be sublinear in time. 

The RTP method has been used as a processing step to the EM algorithm
in crowdsourcing \citep{zhang2014spectral}, but only convergence
properties are considered, which is not enough to provide regret guarantees.

On the theoretical side, our contribution generalizes the setting of
{\em clustered bandits} \citep{maillard2014latent, gentile2014online}
in which a hard clustering model is assumed (one user is assigned to
one class, or equivalently mixture distributions can only have support
size $1$). In particular, \citet{maillard2014latent} specifically
highlight the benefit of a collaborative gain across users against
using a vanilla UCB for each user. However their setting is less
general than assuming a soft clustering of users (one user corresponds
to a mixture of classes) across various ``representative'' taste
profiles as we study here.

The Alternating Least-Squares (ALS) method
\citep{takacs2012alternating,mary2014bandits} has been shown to yield
promising experimental results in similar settings where both $U$ and
$V$ are unknown. However, no theoretical guarantees are known for this
algorithm that may converge to a local optimum in general.

The work of \citet{ValMunKve14:spectral} studies
stochastic bandits with a linear model over a low-rank (graph
Laplacian) structure. However, they assume complete knowledge of the
graph and hence knowledge of the eigenvectors of the Laplacian,
converting it into a bilinear problem with only \textit{one}-sided
uncertainty. This is in contrast to our setup where {\em both} $U$,
$V$ are completely uncertain.

Perhaps the closest work to ours is that of
\citet{DjoKraCev13:GPbandits} where the authors develop a flexible
approach for bandit problems in high dimension but with
low-dimensional reward dependence. They use a two-phase algorithm:
First a low-rank matrix completion technique (the Dantzig selector)
estimates the feature-reward map, then a Gaussian Process-UCB (GP-UCB)
bandit algorithm controls the regret, and show that if after $n$ iterations
the approximation error between the feature matrix and its estimate is
less then $\eta$, the final regret is given by the sum of the regret
of GP-UCB when given perfect knowledge of the features and of
$n + \eta(T-n)$ (due to the learning phase and approximation
error). This results in an overall regret scaling with $O(T^{4/5})$.
We depart from their results in two fundamental ways: Firstly, they
have the possibility of uniformly sampling the entries (a common
assumption in low-rank matrix completion techniques). We do not have
this luxury in our setting as we do not control the process of user
arrivals, that is not constrained to be uniform.  Secondly, we prove
and exploit a novel robustness property (see
Theorem~\ref{thm:UpOFUL2}) of the bandit subroutine we use (\OFUL\ in
our case instead of GP-UCB), which allows us to effectively eliminate
the approximation error in their work and obtain a $O(\sqrt{T})$
regret bound (see Theorem~\ref{thm:main}).

\section{{Conclusion \& Directions}}

We consider a full-blown latent class mixture model in
which users are described by unknown mixtures across unknown user
classes, more general and challenging than when
users are assumed to fall perfectly in one class
\citep{gentile2014online,maillard2014latent}.

We provide the first provable sublinear regret
guarantees in this setting, when both the canonical classes and user
mixture weights are completely unknown, which we believe is striking
when compared to existing work in the setting, e.g.,
alternate minimization typically gets stuck in local minima.  We
currently use a combination of noisy tensor factorization and linear
bandit techniques, and control the uncertainty in the estimates
resulting from each one of these techniques. This enables us to
effectively recover the latent class
structure.

Future directions include reducing the numerical constant (e.g. using
an alternative to RTP), and studying how to combine our work with the
aggregation of user parameters suggested in
\citet{maillard2014latent}.

\bibliography{localbiblio}

\begin{thebibliography}{25}
\providecommand{\natexlab}[1]{#1}
\providecommand{\url}[1]{\texttt{#1}}
\expandafter\ifx\csname urlstyle\endcsname\relax
  \providecommand{\doi}[1]{doi: #1}\else
  \providecommand{\doi}{doi: \begingroup \urlstyle{rm}\Url}\fi

\bibitem[Abbasi-Yadkori et~al.(2011)Abbasi-Yadkori, Pal, and
  Szepesvari]{AbbPalCsa11:linbandits}
Yasin Abbasi-Yadkori, David Pal, and Csaba Szepesvari.
\newblock {Improved Algorithms for Linear Stochastic Bandits}.
\newblock In \emph{Proc. NIPS}, pages 2312--2320, 2011.

\bibitem[Anandkumar et~al.(2012)Anandkumar, Hsu, and
  Kakade]{anandkumar2012method}
Animashree Anandkumar, Daniel Hsu, and Sham~M Kakade.
\newblock A method of moments for mixture models and hidden markov models.
\newblock \emph{arXiv preprint arXiv:1203.0683}, 2012.

\bibitem[Anandkumar et~al.(2014{\natexlab{a}})Anandkumar, Ge, Hsu, and
  Kakade]{AnaGeHsuKak14:tensormixedmodels}
Animashree Anandkumar, Rong Ge, Daniel Hsu, and Sham~M Kakade.
\newblock A tensor approach to learning mixed membership community models.
\newblock \emph{Journal of Machine Learning Research}, 15\penalty0
  (1):\penalty0 2239--2312, 2014{\natexlab{a}}.

\bibitem[Anandkumar et~al.(2014{\natexlab{b}})Anandkumar, Ge, Hsu, Kakade, and
  Telgarsky]{Anandkumar}
Animashree Anandkumar, Rong Ge, Daniel Hsu, Sham~M. Kakade, and Matus
  Telgarsky.
\newblock Tensor decompositions for learning latent variable models.
\newblock \emph{J. Mach. Learn. Res.}, 15\penalty0 (1):\penalty0 2773--2832,
  January 2014{\natexlab{b}}.

\bibitem[Auer et~al.(2002)Auer, Cesa-Bianchi, and
  Fischer]{AuerEtAl02FiniteTime}
P.~Auer, N.~Cesa-Bianchi, and P.~Fischer.
\newblock {Finite-time analysis of the multiarmed bandit problem}.
\newblock \emph{Machine Learning}, 47\penalty0 (2):\penalty0 235--256, 2002.

\bibitem[Bresler et~al.(2014)Bresler, Chen, and Shah]{BreCheSha14:latentCF}
Guy Bresler, George~H Chen, and Devavrat Shah.
\newblock A latent source model for online collaborative filtering.
\newblock In \emph{Proc. NIPS 27}, pages 3347--3355. Curran Associates, Inc.,
  2014.

\bibitem[Dani et~al.(2008)Dani, Hayes, and Kakade]{Dani2008Lin}
Varsha Dani, Thomas~P. Hayes, and Sham~M. Kakade.
\newblock {Stochastic Linear Optimization under Bandit Feedback}.
\newblock In \emph{Proc. COLT}, 2008.

\bibitem[Djolonga et~al.(2013)Djolonga, Krause, and
  Cevher]{DjoKraCev13:GPbandits}
Josip Djolonga, Andreas Krause, and Volkan Cevher.
\newblock High-dimensional {G}aussian process bandits.
\newblock In \emph{Proc. NIPS}, pages 1025--1033, 2013.

\bibitem[Forsgren(1996)]{forsgren1996linear}
Anders Forsgren.
\newblock On linear least-squares problems with diagonally dominant weight
  matrices.
\newblock \emph{SIAM Journal on Matrix Analysis and Applications}, 17\penalty0
  (4):\penalty0 763--788, 1996.

\bibitem[Gentile et~al.(2014)Gentile, Li, and Zappella]{gentile2014online}
Claudio Gentile, Shuai Li, and Giovanni Zappella.
\newblock Online clustering of bandits.
\newblock In \emph{Proc. ICML}, pages 757--765, 2014.

\bibitem[Gheshlaghi~Azar et~al.(2013)Gheshlaghi~Azar, Lazaric, and
  Brunskill]{transfer}
Mohammad Gheshlaghi~Azar, Alessandro Lazaric, and Emma Brunskill.
\newblock Sequential transfer in multi-armed bandit with finite set of models.
\newblock In \emph{Proc. NIPS}, pages 2220--2228. Curran Associates, Inc.,
  2013.

\bibitem[Hofmann and Puzicha(1999)]{HofPuz99:latent}
Thomas Hofmann and Jan Puzicha.
\newblock Latent class models for collaborative filtering.
\newblock In \emph{IJCAI}, volume~99, pages 688--693, 1999.

\bibitem[Jain et~al.(2013)Jain, Netrapalli, and Sanghavi]{jain2013low}
Prateek Jain, Praneeth Netrapalli, and Sujay Sanghavi.
\newblock Low-rank matrix completion using alternating minimization.
\newblock In \emph{Proc. ACM Symposium on Theory Of computing (STOC)}, pages
  665--674. ACM, 2013.

\bibitem[Kleinberg and Sandler(2004)]{KleSan04:mixturemodels}
Jon Kleinberg and Mark Sandler.
\newblock Using mixture models for collaborative filtering.
\newblock In \emph{Proc. ACM Symposium on Theory Of Computing (STOC)}, pages
  569--578. ACM, 2004.

\bibitem[Koc{\'a}k et~al.(2014)Koc{\'a}k, Neu, Valko, and
  Munos]{kocak2014efficient}
Tom{\'a}{\v{s}} Koc{\'a}k, Gergely Neu, Michal Valko, and R{\'e}mi Munos.
\newblock Efficient learning by implicit exploration in bandit problems with
  side observations.
\newblock In \emph{Proc. NIPS}, pages 613--621, 2014.

\bibitem[Lazaric et~al.(2013)Lazaric, Brunskill, et~al.]{lazaric2013sequential}
Alessandro Lazaric, Emma Brunskill, et~al.
\newblock Sequential transfer in multi-armed bandit with finite set of models.
\newblock In \emph{Proc. NIPS}, pages 2220--2228, 2013.

\bibitem[Maillard and Mannor(2014)]{maillard2014latent}
Odalric-Ambrym Maillard and Shie Mannor.
\newblock Latent bandits.
\newblock In \emph{Proc. ICML}, pages 136--144, 2014.

\bibitem[Mary et~al.(2014)Mary, Gaudel, and Philippe]{mary2014bandits}
J{\'e}r{\'e}mie Mary, Romaric Gaudel, and Preux Philippe.
\newblock Bandits warm-up cold recommender systems.
\newblock \emph{arXiv preprint arXiv:1407.2806}, 2014.

\bibitem[Nguyen et~al.(2014)Nguyen, Mary, and Preux]{nguyen2014cold}
Hai~Thanh Nguyen, J{\'e}r{\'e}mie Mary, and Philippe Preux.
\newblock Cold-start problems in recommendation systems via contextual-bandit
  algorithms.
\newblock \emph{arXiv preprint, arXiv:1405.7544}, 2014.

\bibitem[Rusmevichientong and Tsitsiklis(2010)]{rusmevichientong2010linearly}
Paat Rusmevichientong and John~N Tsitsiklis.
\newblock Linearly parameterized bandits.
\newblock \emph{Mathematics of Operations Research}, 35\penalty0 (2):\penalty0
  395--411, 2010.

\bibitem[Stewart et~al.(1990)Stewart, Sun, and Jovanovich]{stewart1990matrix}
Gilbert~W Stewart, Ji-guang Sun, and Harcourt~Brace Jovanovich.
\newblock \emph{Matrix perturbation theory}, volume 175.
\newblock Academic press New York, 1990.

\bibitem[Sutskever et~al.(2009)Sutskever, Tenenbaum, and
  Salakhutdinov]{SutTenSal09:Bayescluster}
Ilya Sutskever, Joshua~B. Tenenbaum, and Ruslan~R Salakhutdinov.
\newblock Modelling relational data using {B}ayesian clustered tensor
  factorization.
\newblock In \emph{Proc. NIPS}, pages 1821--1828. Curran Associates, Inc.,
  2009.

\bibitem[Tak{\'a}cs and Tikk(2012)]{takacs2012alternating}
G{\'a}bor Tak{\'a}cs and Domonkos Tikk.
\newblock Alternating least squares for personalized ranking.
\newblock In \emph{ACM Conference on Recommender systems}, 2012.

\bibitem[Valko et~al.(2014)Valko, Munos, Kveton, and
  Koc{\'a}k]{ValMunKve14:spectral}
Michal Valko, R{\'e}mi Munos, Branislav Kveton, and Tom{\'a}{\v{s}} Koc{\'a}k.
\newblock {Spectral Bandits for Smooth Graph Functions}.
\newblock In \emph{Proc. ICML}, 2014.

\bibitem[Zhang et~al.(2014)Zhang, Chen, Zhou, and Jordan]{zhang2014spectral}
Yuchen Zhang, Xi~Chen, Dengyong Zhou, and Michael~I Jordan.
\newblock Spectral methods meet {EM}: A provably optimal algorithm for
  crowdsourcing.
\newblock In \emph{Proc. NIPS}, pages 1260--1268, 2014.

\end{thebibliography}

\appendix

\section{Proofs of Lemmas~\ref{lem:Mestimates} and \ref{lem:estimates}}

{\bf Proof of Lemma~\ref{lem:Mestimates}}
This result holds by construction of the estimates $\tilde r_{a,a',n}$
and $\tilde r_{a,a',a'',n}$.
Note that

\beqan
\Esp\big[\tilde r_{a,a',n}\big] &=& \frac{1}{n}\sum_{i=1}^n\sum_{b\in[B]}
\Esp\bigg[
\frac{X_{a_{i,1},b_i,i,1}X_{a_{i,2},b_i,i,2}}{p_i(a,a'|b_i)} \indic{a_{i,1}=a,a_{i,2}=a'}\bigg|b_i=b\bigg]\beta(b)\\
&=&
\frac{1}{n}
\sum_{i=1}^n\sum_{b\in[B]}
\sum_{c\in[C]}
\Esp\bigg[
\frac{X_{a_{i,1},b,i,1}X_{a_{i,2},b,i,2}}{p_i(a,a'|b)} \indic{a_{i,1}=a,a_{i,2}=a'}\bigg|b_i=b,c_i=c\bigg]v_{b,c}\beta(b)\\
&=&
\frac{1}{n}
\sum_{i=1}^n\sum_{b\in[B]}
\sum_{c\in[C]}
\Esp\bigg[
X_{a,b,i,1}X_{a',b,i,2}\bigg|b_i=b,c_i=c\bigg]v_{b,c}\beta(b)\\
&\stackrel{(a)}{=}&
\frac{1}{n}
\sum_{i=1}^n\sum_{b\in[B]}
\sum_{c\in[C]}
\Esp\bigg[
X_{a,b,i,1}\bigg|b_i=b,c_i=c\bigg]\Esp \bigg[X_{a',b,i,2}\bigg|b_i=b,c_i=c\bigg]v_{b,c}\beta(b)\\
&=&
\frac{1}{n}
\sum_{i=1}^n\sum_{b\in[B]}
\sum_{c\in[C]}u_{a,c}u_{a',c}v_{b,c}\beta(b)\\
&=&\sum_{c\in[C]}(\sum_{b\in[B]}v_{b,c}\beta(b))u_{a,c}u_{a',c}\\
&=&\sum_{c\in[C]}v_{\beta,c}u_{a,c}u_{a',c}\,,
\eeqan
where $(a)$ holds by independence of the sample generated 
by user $b$ when in the same class $c$. Note that $c_i$ is the same for all $\ell=1,2,3$ interaction steps, that is $c_i=c_{i,1}=c_{i,2}=c_{i,3}$, where $c_{i,\ell}$ is the class corresponding to sample $X_{a,b,i,\ell}$.
This is the reason why we get $u_{a,c}u_{a',c}v_{b,c}$
and not a product $u_{a,c}u_{a',c}v_{b,c}^2$ for instance.

\bigskip
{\bf Proof of Lemma~\ref{lem:estimates}}
Since the rewards generated by each source $a,b$ are i.i.d., the estimate $\tilde r_{a,a',n}$ is a sum of i.i.d. random variables bounded in $[0,1]$, re-weighted by the probability weights $p_i(a,a'|b_i)$, which are measurable functions of the past. Assuming that there exists some
deterministic $q_{2,i}>0$ such that
$\forall  i \in\Nat,  p_i(a,a'|b_i)\geq q_{2,i}$, we can thus apply a version of 
Azuma-Hoeffding inequality for bounded martingale difference sequence.
Let us recall that by this inequality, for a deterministic time $s$, and $(Y_m)_{m\leq s}\in[0,1]$ being a bounded martingale difference sequence,  then for all $\delta\in(0,1)$  it holds
\beqan
\Pr( |\frac{1}{s} \sum_{i=1}^s  Y_i|  \geq \sqrt{\frac{\log(2/\delta)}{2s}}) \leq \delta\,.
\eeqan

In our case, $Y_i = \frac{X_{a_{i,1},b_i,i,1}X_{a_{i,2},b_i,i,2}}{p_i(a,a'|b_i)} \indic{a_{i,1}=a,a_{i,2}=a'}
- m_{a,a'}$, and we deduce that
\beqan
\Pr( |\tilde r_{a,a',n}-m_{a,a'}|  \geq \sqrt{\sum_{i=1}^nq_{2,i}^{-2}\frac{\log(2/\delta)}{2n^2}}) \leq \delta\,.
\eeqan

Likewise, we get that
\beqan
\Pr( |\tilde r_{a,a',a'',n}-m_{a,a',a''}|  \geq \sqrt{\sum_{i=1}^nq_{3,i}^{-2}\frac{\log(2/\delta)}{2n^2}}) \leq \delta\,.
\eeqan

Taking a union bound over the actions in each case, and then over the two events concludes the proof.
$\square$

\bigskip
{\bf Proof of Corollary~\ref{cor:estimates}}
From Lemma~\ref{lem:estimates},  we deduce that on an event of probability higher than $1-\delta$, it holds simultaneously that
\beqan
e_n^{(2)} &\eqdef&	\|\hat M_{n,2}  - M_{2}\| \leq A\sqrt{\sum_{m=1}^nq_{2,m}^{-2}\frac{\log(4A^2/\delta)}{2n^2}})
\,\, \text{ and }\,\,\\
e_n^{(3)} &\eqdef&	\|\hat M_{n,3}  - M_{3}\| \leq A^{3/2}\sqrt{\sum_{m=1}^nq_{3,m}^{-2}\frac{\log(4A^3/\delta)}{2n^2}})\,.
\eeqan

This indeed holds by relating the norm of the matrix (tensor) with each of the elements.
We conclude by replacing the values of $q_{2,i}$ and $q_{3,i}$.
\section{Proof of Theorem~\ref{thm:Uestimates}}

We prove in this section a slightly more detailed result, namely, the following:

{\bf Theorem~\ref{thm:Uestimates}.}
\textit{
Assume that $\{\gamma_i\}_{i\geq 1}$ are chosen such that $n^{-2}\sum_{i=1}^n\gamma_i^{-2}\stackrel{n}{\to} 0$.
Let $\lambda_{\min}$
be the minimum robust eigenvalue of the tensor $T=M_3(W,W,W)$.
Let $\delta\in(0,1)$. Provided that 
\beqan
\frac{n^2}{\sum_{i=1}^n\gamma_i^{-2}} \geq
\max\bigg\{ \frac{2A^6\log(4A^2/\delta)}{\min\{\Gamma,\sigma_{\min}\}^2},
\frac{A^9(1 + 10(\frac{1}{\Gamma} + \frac{1}{\sigma_{\min}})(1+u_{\max}^3))^2C^{5}\log(4A^3/\delta)}{2C^2_1\lambda_{min}^2\sigma_{min}^{3}}
\bigg\}\,,
\eeqan
with probability higher than $1-2\delta$,
there exists some permutation $\pi \in \mathbb{S}_C$ such that for all $c\in[C]$, 
\beqan	
	||u_c - \ol u_{n,\pi(c)}|| 
	\leq   \Delta 
	A^{3}\sqrt{\sum_{i=1}^n\gamma_i^{-2}\frac{C\log(4A^3/\delta)}{2n^2}} + o(n^{-2}\sum_{i=1}^n\gamma_i^{-2}), 
\eeqan
where we introduced the problem-dependent constant
\beqan
\Delta &=& 	
		13\sqrt{\sigma_{max}}\Big(\frac{CA}{\sigma_{min}}\Big)^{3/2}
		\Big(1+10(\frac{1}{\Gamma} + \frac{1}{\sigma_{\min}})(1+u_{\max}^3)\Big)
		+\big(\frac{2\sigma_{max}}{\Gamma}+ \frac{1}{\sigma_{max}}\big)	\frac{1}{v_{\min}^2}\,.
\eeqan
For general $\{\gamma_i\}_{i\geq 1}$ (not necessarily such that $n^{-2}\sum_{i=1}^n\gamma_i^{-2}\stackrel{n}{\to} 0$), it holds with same probability that
\beqan
	||{\bf u}_c - \ol {\bf u}_{n,\pi(c)}||  \leq 
\Diamond 
	A^{3}\sqrt{\sum_{i=1}^n\gamma_i^{-2}\frac{C\log(4A^3/\delta)}{2n^2}}\,,
\eeqan
where, using the notation $\aleph = 1+10(\frac{1}{\Gamma} + \frac{1}{\sigma_{\min}})(1+u_{\max}^3)$, we have introduced the constant
\beqan
\lefteqn{
\Diamond = \Big(\frac{CA}{\sigma_{min}}\Big)^{3/2}
	\bigg(	13\sqrt{\sigma_{max}} + 4
	 \sqrt{2\min\{\Gamma,\sigma_{\min}\}}
	 +5\big(\frac{\sigma_{max}}{\Gamma}+ 
	 \frac{1}{2\sigma_{max}}\big)
	  \min\{\Gamma,\sigma_{\min}\}
	  \bigg)
	 \aleph}\\
	  &&
	  \hspace{-3mm}+\,\, \big(\frac{2\sigma_{max}}{\Gamma}+ \frac{1}{\sigma_{max}}\big)	\frac{1}{v_{\min}^2}
	  + 5\sqrt{3/8}\Big(\sqrt{\sigma_{\max}} + 
	  	 \sqrt{ \min\{\Gamma,\sigma_{\min}\}/2}\Big)
	  	 \big(\frac{2CA}{\sigma_{\min}}\big)^3 \aleph^2
	  	 \min\{\Gamma,\sigma_{\min}\}\,.
\eeqan
}
%

\begin{proof}
The proof closely follows that of \citet{transfer}.
First, note that by property of the rank $1$ decomposition (\citep[Theorem~4.3]{Anandkumar}), it holds that 
 $\lambda_c = (\sum_{b\in[B]} v_{b,c}\beta(b))^{-2}$ and thus $v_{\min}^{-2} \geq \lambda_{\max} \geq \lambda_{\min} \geq 1$.

We first decompose the following term to make appear
the terms from Proposition~\ref{prop:RTP}:
\beqa	
\lefteqn{||{\bf u}_c - \ol{{\bf u}}_{n,\pi(c)}||  \leq} \label{eqn:uDiffdec}\\
&&\underbrace{|\lambda_c -\hat \lambda_{n,\pi(c)}|}_{RTP.1} \underbrace{\|{W^\top}^\dagger\|}_{b} \underbrace{|| \phi_{c}||}_{a}
+  \underbrace{|\hat \lambda_{\pi(c)}|}_{RTP.3}
\underbrace{\|{W^\top}^\dagger - {\hat{W}}^\top{}^\dagger\|}_{d}
\underbrace{|| \phi_{c}||}_{1}
+  \underbrace{|\hat \lambda_{\pi(c)}|}_{RTP.3}
\underbrace{\| {\hat{W}}^\top{}^\dagger\|}_{c}
\underbrace{|| \phi_c - \hat \phi_{n,\pi(c)}||}_{RTP.2}\,.
\nonumber
\eeqa

Note that $\phi_c$, and $\hat \phi_{n,\pi(c)}$ are both normalized vectors. 
Thus, $(a)$ is bounded as $|| \phi_{c}||\leq 1$.
It holds for $(b)$
that $\|{W^\top}^\dagger\| \leq \sqrt{C \sigma_{max}}$, and for $(c)$, on the $1-\delta$ event $\Omega$ from Corollary~\ref{cor:estimates}, that
\beqa
\| {\hat{W}}^\top{}^\dagger\| \leq \sqrt{C \hat \sigma_{max}} \leq \sqrt{C}(\sqrt{\sigma_{max}} + \sqrt{e_n^{(2)}})\,.
\label{eqn:W}
\eeqa
The term $(d)$ requires a little more work. It holds that 
\beqan
\|{W^\top}^\dagger - {\hat{W}}^\top{}^\dagger\| &=& \|\hat U \hat D^{1/2} - U D^{1/2}\|\\
	&\leq&  \|(\hat U - U)D^{1/2}\| + \|\hat U(\hat D^{1/2} - D^{1/2})\|\\
	&\leq& \underbrace{\|\hat U - U\|}_{e}\sigma_{max} +  \underbrace{\|\hat D^{1/2} - D^{1/2}\|}_{f}\sqrt{C}\,.
\eeqan

We use the result of Lemma~5 from \citet{transfer} to control $(e)$ and $(f)$.
If $e_n^{(2)} \leq \frac{1}{2} \Gamma$, then it holds
\beqan
\|\hat D^{1/2} - D^{1/2}\| \leq  \frac{e_n^{(2)} }{\sigma_{\max}} \qquad
\|\hat U - U\| \leq \frac{2\sqrt{C}e_n^{(2)}}{\Gamma}\,,
\eeqan
from which we deduce that 
\beqa
\|{W^\top}^\dagger - {\hat{W}}^\top{}^\dagger\| \leq  (\frac{2\sigma_{\max}}{\Gamma} +\frac{1}{\sigma_{\max}})\sqrt{C}e_n^{(2)}\,.
\label{eqn:Wdiff}
\eeqa

At this point, $(RTP.1), (RTP.2)$ and $(RTP.3)$ are controlled by the perturbation method from \citet{Anandkumar},
under the condition that $e_n = \| T - \hat T \| \leq C_1\frac{\lambda_{min}}{C}$ (where $C_1$ is a universal constant). In this case, with probability $1-\delta$, the RTP algorithm with well-chosen parameters achieves
\beqan
|\lambda_c -\hat \lambda_{n,\pi(c)}| &\leq& 5\| T - \hat T_n \|\\
\| \phi_c - \hat  \phi_{n,\pi(c)}\| &\leq& 8\frac{\| T - \hat T_n \|}{\lambda_c}\,.
\eeqan

In order to make the condition explicit in our setting, we use the fact that by Lemma~6 from \citet{transfer}, if $e_n^{(2)} \leq  \frac{1}{2} \min\{\Gamma,\sigma_{\min}\}$ then 
\beqa
e_n \leq \Big(\frac{C}{\sigma_{\min}}\Big)^{3/2}\bigg(e_n^{(3)} + 2(1+\sqrt{2}+2)e_n^{(2)} (\frac{1}{\Gamma_\sigma} + \frac{1}{\sigma_{min}})(e_n^{(3)} + \max_c||{\bf u}_c||^3)\bigg)\,.
\label{eqn:en}
\eeqa

The condition $e_n^{(2)} \leq  \frac{1}{2} \min\{\Gamma,\sigma_{\min}\}$
 holds if the number of sessions $n$ is sufficiently large:
Indeed on an event  of probability higher than $1-\delta$, then
it is enough that
\beqan
A^3\sqrt{\sum_{i=1}^n\gamma_i^{-2}\frac{\log(4A^2/\delta)}{2n^2}}\leq \frac{1}{2} \min\{\Gamma,\sigma_{\min}\}\,, 
\eeqan
that is, reordering the terms, that
\beqa
\frac{n^2}{\sum_{i=1}^n\gamma_i^{-2}} \geq \frac{2A^6\log(4A^2/\delta)}{\min\{\Gamma,\sigma_{\min}\}^2}\,. 
\label{eqn:nbig1}
\eeqa

Now, in order to satisfy the condition $e_n = \| T - \hat T_n \| \leq C_1\frac{\lambda_{\min}}{C}$, it is enough that
\beqan
\Big(\frac{C}{\sigma_{\min}}\Big)^{3/2}\bigg(e_n^{(3)} + 2(1+\sqrt{2}+2)e_n^{(2)} (\frac{1}{\Gamma} + \frac{1}{\sigma_{\min}})(e_n^{(3)} + \max_c||u_c||^3)\bigg) \leq C_1\frac{\lambda_{\min}}{C}\,.
\eeqan
Let us decompose the left-hand-side term: After some simplifications using $ \max_c||{\bf u}_c||^3 \leq A^{3/2}u_{\max}^3$ and $e_n^{(3)}\leq A^{3/2}$, the previous inequality happens when
\beqan
e_n^{(3)} + A^{3/2}\varhexagon e_n^{(2)} \leq C_1\frac{\lambda_{min}\sigma_{min}^{3/2}}{C^{5/2}}\,.
\eeqan
where $\varhexagon = 2(1+\sqrt{2}+2)(\frac{1}{\Gamma} + \frac{1}{\sigma_{\min}})(1+u_{\max}^3)$.
Using the definition of $e_n^{(3)}$ and $e_n^{(2)}$
then we deduce that it is enough that
\beqan
(1+\varhexagon)A^{9/2}\sqrt{\sum_{i=1}^n\gamma_i^{-2}\frac{\log(4A^3/\delta)}{2n^2}} \leq C_1\frac{\lambda_{min}\sigma_{min}^{3/2}}{C^{5/2}}\,,
\eeqan
that is, reordering the terms that
\beqa
\frac{n^2}{\sum_{i=1}^n\gamma_i^{-2}}\geq \frac{A^9(1 + \varhexagon)^2C^{5}\log(4A^3/\delta)}{2C^2_1\lambda_{min}^2\sigma_{min}^{3}}\,.\label{eqn:nbig2}
\eeqa

Combining the decomposition \eqref{eqn:uDiffdec} with 
\eqref{eqn:W},\eqref{eqn:Wdiff}, 
 and using the fact that $v_{\min}^{-2}\geq \lambda_c\geq 1$, we obtain
\beqan	
||{\bf u}_c - \bar {\bf u}_{n,\pi(c)}||  &\leq & 5 e_n \sqrt{C}\sqrt{\sigma_{\max}}
+  (\lambda_c+5 e_n)\sqrt{C}\big(\frac{2\sigma_{\max}}{\Gamma}+ \frac{1}{\sigma_{\max}}\big)e_n^{(2)}\\
&& + 8\sqrt{C}(\lambda_c+5 e_n)(\sqrt{\sigma_{\max}} + \sqrt{e_n^{(2)}}) \frac{e_n}{\lambda_c}\,.\\
&\leq&  
\sqrt{C}\bigg[ 13\sqrt{\sigma_{\max}}e_n  + \big(\frac{2\sigma_{\max}}{\Gamma}+ \frac{1}{\sigma_{\max}}\big)\frac{e_n^{(2)}}{v_{\min}^2} + 8\sqrt{e_n^{(2)}} e_n\\
&&+ 5\big(\frac{2\sigma_{\max}}{\Gamma}+ \frac{1}{\sigma_{\max}}\big)e_n^{(2)}e_n
  +  40(\sqrt{\sigma_{\max}} + \sqrt{e_n^{(2)}})e_n^2\bigg]\,.
\eeqan
Now, using \eqref{eqn:en} and unfolding the last inequality, it holds with probability higher than $1-2\delta$ that
\beqan
\lefteqn{||{\bf u}_c - \bar{\bf u}_{n,\pi(c)}||}\\
	&\leq& \sqrt{C}\bigg[ 
	13\sqrt{\sigma_{max}}\Big(\frac{C}{\sigma_{min}}\Big)^{3/2}(e_n^{(3)}+ e_n^{(2)}A^{3/2}\varhexagon) +\big(\frac{2\sigma_{max}}{\Gamma_\sigma}+ \frac{1}{\sigma_{max}}\big)\frac{e_n^{(2)}}{v_{\min}^2}\\
	&& +8\Big(\frac{C}{\sigma_{min}}\Big)^{3/2}\sqrt{e_n^{(2)}}(e_n^{(3)}+ e_n^{(2)}A^{3/2}\varhexagon)\\
	&&		 +5\Big(\frac{C}{\sigma_{min}}\Big)^{3/2}\big(\frac{2\sigma_{max}}{\Gamma_\sigma}+ \frac{1}{\sigma_{max}}\big)e_n^{(2)}(e_n^{(3)}+ e_n^{(2)}A^{3/2}\varhexagon)		 +40(\sqrt{\sigma_{max}} + \sqrt{e_n^{(2)}})e_n^2
	\bigg]\\
	&\leq&  
	\bigg[
	13\sqrt{\sigma_{max}}\Big(\frac{CA}{\sigma_{min}}\Big)^{3/2}
	(1+\varhexagon)
	+\big(\frac{2\sigma_{max}}{\Gamma_\sigma}+ \frac{1}{\sigma_{max}}\big)	\frac{1}{v_{\min}^2}\bigg]\\
	&&\times A^{3}\sqrt{\sum_{i=1}^n\gamma_i^{-2}\frac{C\log(4A^3/\delta)}{2n^2}}	+ o(n^{-2}\sum_{i=1}^n\gamma_i^{-2})\,,\\
\eeqan
which, after some cosmetic simplifications, concludes the first part of the proof of Theorem~\ref{thm:Uestimates}.

Alternatively, when 
$n^{-2}\sum_{i=1}^n\gamma_i^{-2}\not\to\infty$, we can always resort to the condition that $e_n^{(2)} \leq 1/2 \min\{\Gamma,\sigma_{\min}\}$ in order to simplify the previous derivation.  We deduce, similarly, that	 
\beqan
\lefteqn{||{\bf u}_c - \bar{\bf u}_{n,\pi(c)}||}\\
	&\leq& \sqrt{C}\bigg[ 
	13\sqrt{\sigma_{max}}\Big(\frac{C}{\sigma_{min}}\Big)^{3/2}(e_n^{(3)}+ e_n^{(2)}A^{3/2}\varhexagon) +\big(\frac{2\sigma_{max}}{\Gamma_\sigma}+ \frac{1}{\sigma_{max}}\big)\frac{e_n^{(2)}}{v_{\min}^2}\\
	&& +8\Big(\frac{C}{\sigma_{min}}\Big)^{3/2}\sqrt{e_n^{(2)}}(e_n^{(3)}+ e_n^{(2)}A^{3/2}\varhexagon)\\
	&&		 +5\Big(\frac{C}{\sigma_{min}}\Big)^{3/2}\big(\frac{2\sigma_{max}}{\Gamma_\sigma}+ \frac{1}{\sigma_{max}}\big)e_n^{(2)}(e_n^{(3)}+ e_n^{(2)}A^{3/2}\varhexagon)		 +40(\sqrt{\sigma_{max}} + \sqrt{e_n^{(2)}})e_n^2
	\bigg]\\
	&\leq&
	\bigg[
	\bigg(
	13\sqrt{\sigma_{max}}\Big(\frac{CA}{\sigma_{min}}\Big)^{3/2}
	 + 8\Big(\frac{CA}{\sigma_{min}}\Big)^{3/2}
	 \sqrt{\min\{\Gamma,\sigma_{\min}\}/2}\\
	 &&
	 +5\Big(\frac{CA}{\sigma_{min}}\Big)^{3/2}\big(\frac{\sigma_{max}}{\Gamma_\sigma}+ 
	 \frac{1}{2\sigma_{max}}\big)
	  \min\{\Gamma,\sigma_{\min}\}
	  \bigg)
	 (1+\varhexagon)
	  + \big(\frac{2\sigma_{max}}{\Gamma_\sigma}+ \frac{1}{\sigma_{max}}\big)	\frac{1}{v_{\min}^2}\\
	  &&
	  + 40\Big(\sqrt{\sigma_{\max}} + 
	  	 \sqrt{ \min\{\Gamma,\sigma_{\min}\}/2}\Big)
	  	 \big(\frac{CA}{\sigma_{\min}}\big)^3 (1+\varhexagon)^2
	  	 \min\{\Gamma,\sigma_{\min}\}\sqrt{3/8}
	  \bigg]\\
	 &&\times
	 A^{3}\sqrt{\sum_{i=1}^n\gamma_i^{-2}\frac{C\log(4A^3/\delta)}{2n^2}}\,,
	 \eeqan
where, in order to control the last term $e_n^2$, we used the property that
\beqan
e_n& \leq& \Big(\frac{CA}{\sigma_{\min}}\Big)^{3/2}\Big(1+\varhexagon \Big)
\min\{e_n^{(2)}\sqrt{\frac{\log(4A^3/\delta)}{\log(4A^2/\delta)}}, A^{-3/2}e_n^{(3)}\}\\
&\leq&  \Big(\frac{CA}{\sigma_{\min}}\Big)^{3/2}\Big(1+\varhexagon \Big)\min\{\sqrt{3/2}e_n^{(2)},A^{-3/2}e_n^{(3)}\}\,.
\eeqan
\end{proof}

\section{Proof of Theorem~\ref{thm:pOFUL2}}

\begin{proof}
  Let $\mathbf{M}_{1:t} = ({\bf m}_{A_1}, \ldots, {\bf m}_{A_t})^\top$. The
  argument used to prove Theorem 2 in Yadkori et al, 2011, can be used
  to show that
  \begin{align*} \hat{\bf v}_{t-1} &= V_{t-1}^{-1}
    \bar{\bf U}_{1:t-1}\eta_{1:t-1} + V_{t-1}^{-1}
    \bar{\bf U}_{1:t-1} \mathbf{M}_{1:t-1} 
  \end{align*}
  where $\eta_{1:t-1} \bydef (\eta_1, \ldots, \eta_{t-1})$ is the
  observed noise sequence. Let $\mathbf{E}_{1:t-1} \bydef
  (\epsilon_{A_1}, \ldots, \epsilon_{A_t})^\top = \mathbf{M}_{1:t-1} -
  \bar{\bf U}_{1:t-1} {\bf v}^\circ$. We then have
  \begin{align*}
    \hat{\bf v}_{t-1} &= V_{t-1}^{-1} \bar{\bf U}_{1:t-1}
    \eta_{1:t-1} + V_{t-1}^{-1}
    \bar{\bf U}_{1:t-1} \mathbf{M}_{1:t-1} \\
    &= V_{t-1}^{-1} \bar{\bf U}_{1:t-1} \eta_{1:t-1} +
    V_{t-1}^{-1} \bar{\bf U}_{1:t-1} \left( \bar{\bf U}_{1:t-1}^\top
      {\bf v}^\circ + \mathbf{E}_{1:t-1}\right) \\
    &= V_{t-1}^{-1} \bar{\bf U}_{1:t-1} \eta_{1:t-1} +
    {\bf v}^\circ - \lambda V_{t-1}^{-1} {\bf v}^\circ + V_{t-1}^{-1}
    \bar{\bf U}_{1:t-1} \mathbf{E}_{1:t-1}.
  \end{align*}

  Thus, letting ${\bf v}_{t-1}^+ \bydef {\bf v}^\circ +
  V_{t-1}^{-1}\bar{U}_{1:t-1} \mathbf{E}_{1:t-1}$ and using the
  above with techniques from Yadkori et al together with
  $\norm{{\bf v}^\circ}_2 \leq R_\Theta$, we have that
  \begin{align*}
    {\bf v}_{t-1}^+ \in \cC_{t-1}
  \end{align*}
  with probability at least $1-\delta$.

  Now, let $a_{t-1}^{+} \in \arg\max_{a \in \mathcal{A}} \bar{\bf u}_a^\top
  {\bf v}_{t-1}^+$ be an optimal action corresponding to the
  approximate parameter ${\bf v}_{t-1}^+$, and define the instantaneous
  regret at time $t$ {\em with respect to the approximate parameter}
  as
  \[ r_t^+ \bydef \bar{\bf u}_{a_{t-1}^{+}}^\top {\bf v}_{t-1}^+ - \bar{\bf u}_{A_t}^\top
  {\bf v}_{t-1}^+ \geq 0.\] We now bound this approximate regret using
  arguments along the lines of Yadkori et al, 2011. Consider
  \begin{align}
    r_t^+ &= \bar{\bf u}_{a_{t-1}^{+}}^\top {\bf v}_{t-1}^+ - \bar{\bf u}_{A_t}^\top
            {\bf v}_{t-1}^+ \nonumber \\
          &\leq \bar{\bf u}_{A_t}^\top \tilde{{\bf v}}_t - \bar{\bf u}_{A_t}^\top {\bf v}_{t-1}^+
            \quad \quad \mbox{(since $(A_t, \tilde{{\bf v}}_t)$ is optimistic)}  \nonumber \\
          &= \bar{\bf u}_{A_t}^\top \left( \tilde{{\bf v}}_t - {\bf v}_{t-1}^+ \right)  \nonumber \\
          &= \bar{\bf u}_{A_t}^\top \left( \tilde{{\bf v}}_t - \hat{{\bf v}}_{t-1}
            \right) + \bar{\bf u}_{A_t}^\top \left( \hat{{\bf v}}_{t-1} - {\bf v}_{t-1}^+
            \right)  \nonumber \\
          &\leq \norm{\bar{\bf u}_{A_t}}_{V_{t-1}^{-1}} \norm{ \tilde{{\bf v}}_t - \hat{{\bf v}}_{t-1}
            }_{V_{t-1}}  + \norm{\bar{\bf u}_{A_t}}_{V_{t-1}^{-1}} \norm{ \hat{{\bf v}}_{t-1} - {\bf v}_{t-1}^+}_{V_{t-1}} \quad \quad
            \mbox{(Cauchy-Schwarz's inequality)}  \nonumber \\
          &\leq 2 {D_{t-1}} \norm{\bar{\bf u}_{A_t}}_{V_{t-1}^{-1}}. \label{eqn:rtplus}
  \end{align}

  Noting that $m_a \in [-1,1]$ $\forall a$, the regret can be
  written as
  \begin{align*}
    R_T &= \sum_{t=1}^T \left( m_{a^\star} - m_{A_t} \right)
          = \sum_{t=1}^T \min\{ m_{a^\star} - m_{A_t}, 2 \} \\
        &= \rho' \sum_{a \neq a^\star} \sum_{t=1}^T \min \left\{ \frac{m_{a^\star} - m_{a}}{\rho'}, \frac{2}{\rho'} \right\} \indic{A_t = a}  \\
        &\leq \rho' \sum_{a \neq a^\star} \sum_{t=1}^T \min \left\{ \bar{\bf u}_{a^\star}^\top
          {\bf v}^\circ - \bar{\bf u}_{a}^\top {\bf v}^\circ , \frac{2}{\rho'} \right\}  \indic{A_t = a} \quad \mbox{(using the definition of $\rho'$)} \\
        &\stackrel{(a)}{\leq} \rho' \sum_{t=1}^T \min \left\{ 2\left( \bar{\bf u}_{a^\star}^\top
          {\bf v}_{t-1}^+ - \bar{\bf u}_{A_t}^\top {\bf v}_{t-1}^+ \right), \frac{2}{\rho'} \right\} \stackrel{(b)}{=} 2 \rho' \sum_{t=1}^T \min \left\{ \bar{\bf u}_{a_{t-1}^+}^\top
          {\bf v}_{t-1}^+ - \bar{\bf u}_{A_t}^\top {\bf v}_{t-1}^+ , \frac{1}{\rho'} \right\} \\
        &= 2 \rho' \sum_{t=1}^T \min \left\{ r_t^+ , \frac{1}{\rho'} \right\} = \rho' \sum_{t=1}^T \frac{2}{\rho'}
          \min \left\{ {\rho' r_t^+} , 1 \right\} \stackrel{(c)}{\leq}  \rho' \sum_{t=1}^T \frac{2}{\rho'}
          \min \left\{ 2 \rho' {D_{t-1}} \norm{\bar{\bf u}_{A_t}}_{V_{t-1}^{-1}} , 1 \right\} \\
        &\stackrel{(d)}{\leq} \rho' \sum_{t=1}^T {4} D_{t-1}
          \min \left\{\norm{\bar{\bf u}_{A_t}}_{V_{t-1}^{-1}} , 1 \right\} \\
        &\leq  \rho' \sqrt{ T \sum_{t=1}^T 16
          {D_{T}}^2  \min \left\{ \norm{\bar{\bf u}_{A_t}}_{V_{t-1}^{-1}}^2 , 1 \right\} } \quad \mbox{(by using Cauchy-Schwarz's inequality)}.
  \end{align*}
  In the derivation above, 
  \begin{itemize}
  \item Steps $(a)$ and $(b)$ hold because of the following. By Lemma
    \ref{lem:parameterbias} (to follow below),
    $\norm{{\bf v}_{t-1}^+ - {\bf v}^\circ}_2 =
    \norm{V_{t-1}^{-1}\bar{\bf U}_{1:t-1} \mathbf{E}_{1:t-1}}_2
    \leq \alpha(\bar U) \norm{\epsilon}_2$.
    Since $\arg \max_{a \in \cA} \bar{\bf u}_a^\top {\bf v}^\circ$ is uniquely
    $a^\star$ by hypothesis, we have, thanks to Lemma
    \ref{lem:criticalepsilon} (to follow below), that
    $\bar{\bf u}_{a^\star}^\top {\bf v}_{t-1}^+ - \bar{\bf u}_a^\top {\bf v}_{t-1}^+ >
    \frac{\bar{\bf u}_{a^\star}^\top {\bf v}^\circ - \bar{\bf u}_a^\top {\bf v}^\circ}{2} >
    0$
    $\forall a \neq a^\star$, establishing $(a)$. This in turn shows that the
    optimal action for ${\bf v}_{t-1}^+$ is uniquely $a^\star$ at all
    times $t$, i.e.,
    $a_{t-1}^{+} = \arg\max_{a \in \mathcal{A}} \bar{\bf u}_a^\top
    {\bf v}_{t-1}^+ = a^\star$, which is equality $(b)$.

  \item Inequality $(c)$  holds by (\ref{eqn:rtplus}) and
   $(d)$ holds because $\rho' \geq 1$ by definition,
    and $D_{t-1} \geq \lambda^{1/2}R_{\Theta} \geq 1/2$ by hypothesis,
    implying that $2 \rho' D_{t-1} \geq 1$.
  \end{itemize}

  The argument from here can be continued in the same way as in
  \citet{AbbPalCsa11:linbandits} to yield
  \begin{align*}
    R_T \leq 8 \rho' \sqrt{TC \log\left(1 +
    \frac{TR_\cX^2}{\lambda C}\right)} \left(\lambda^{1/2}R_\Theta +
    R\sqrt{2 \log\frac{1}{\delta} + C \log
    \left(1+\frac{TR_\cX^2}{\lambda C}\right)} \right).
  \end{align*}

This proves the theorem.

\end{proof}
    
\begin{mylemma}[Analysis of the time-varying parameter error
  $V_{t-1}^{-1}\bar{\bf U}_{1:t-1} \mathbf{E}_{1:t-1}$]
    \label{lem:parameterbias}
    Let $\epsilon_a = m_a - \bar{\bf u}_a^\top {\bf v}^\circ$ be the bias in
    arm $a$'s reward due to model error, and let $\epsilon \equiv
    (\epsilon_a)_{a \in \mathcal{A}}$ be the $|\mathcal{A}|$
    dimensional vector of arm reward biases. Then,
    \[ \norm{V_{t-1}^{-1}\bar{\bf U}_{1:t-1} \mathbf{E}_{1:t-1}}_2
    \leq \left( \max_{J} \norm{\mathbf{A}_J^{-1}}_2 \right)
    \norm{\epsilon}_2, \] where $\mathbf{A}_{(A+C) \times
      C} = \left[ \begin{array}{c}
        \bar{U} \\
        I_{d} \end{array} \right]$, $\mathbf{A}_J$ is the $C \times C$
    submatrix of $\mathbf{A}$ formed by picking rows $J$, and $J$
    ranges over all subsets of full-rank rows of $\mathbf{A}$.
  \end{mylemma}

  \begin{proof}[Proof of Lemma \ref{lem:parameterbias}]
  Let $z_{t-1} \bydef V_{t-1}^{-1}\bar{\bf U}_{1:t-1}
  \mathbf{E}_{1:t-1} = {\bf v}_{t-1} - {\bf v}^\circ \in \mathbb{R}^C$,
  with $\norm{{\bf E}_{1:t-1}}_\infty \leq \norm{\epsilon}_\infty = \norm{{\bf m} -
    \bar{U} {\bf v}^\circ}_\infty$. We have
  \begin{align*}
    z_{t-1} &= \left( \sum_{s=1}^{t-1} \bar{\bf u}_{A_s} \bar{\bf u}_{A_s}^\top + \lambda
      I \right)^{-1} \sum_{s=1}^{t-1} \epsilon_{A_s} \bar{\bf u}_{A_s} \\
    &= \left( \frac{1}{t-1}\sum_{s=1}^{t-1} \bar{\bf u}_{A_s} \bar{\bf u}_{A_s}^\top +
      \frac{\lambda}{t-1} I \right)^{-1} \frac{1}{t-1}\sum_{s=1}^{t-1}
    \epsilon_{A_s} \bar{\bf u}_{A_s} \\
    &= \left( \sum_{a \in \mathcal{A}} \bar{\bf u}_{a} \bar{\bf u}_{a}^\top
      \frac{\sum_{s=1}^{t-1} \indic{A_s = a}}{t-1} +
      \frac{\lambda}{t-1} I \right)^{-1} \sum_{a \in \mathcal{A}}
    \epsilon_{a} \bar{\bf u}_{a} \frac{\sum_{s=1}^{t-1} \indic{A_s = a}}{t-1} \\
    &= \left( \sum_{a \in \mathcal{A}} \bar{\bf u}_{a} \bar{\bf u}_{a}^\top f_a(t-1) +
      \frac{\lambda}{t-1}I \right)^{-1} \sum_{a \in \mathcal{A}}
    \epsilon_{a} \bar{\bf u}_{a} f_a(t-1),
  \end{align*}
  where $f_a(t-1) \equiv f_a$ represents the empirical frequency with
  which action $a \in \mathcal{A}$ has been played up to and including
  time $t-1$. This allows us to equivalently interpret $z_{t-1}$ as
  the solution of a {\em weighted} $\ell^2$-regularized least squares
  regression problem with $K = |\mathcal{A}|$ observations (instead of
  the original interpretation with $t-1$ observations) as follows.

  Let $\mathbf{F}^{1/2}$ be the $A \times A$ diagonal matrix with the
  values $\sqrt{f_1}, \ldots, \sqrt{f_A}$ on the diagonal (note:
  $\sum_{a=1}^A f_a = 1$). With this, we can express $z_{t-1}$ as
  \begin{align*}
    z_{t-1} &= \arg\min_{z \in \mathbb{R}^C} \norm{\mathbf{F}^{1/2}
      \bar{U} z - \mathbf{F}^{1/2} \epsilon }_2^2 +
    \frac{\lambda}{t-1}\norm{z}_2^2 \\
    &= \arg\min_{z \in \mathbb{R}^C} \norm{\mathbf{F}^{1/2}\left(
        \bar{U} z - \epsilon \right) }_2^2 +
    \frac{\lambda}{t-1}\norm{z}_2^2 \\
    &= \arg\min_{z \in \mathbb{R}^C} \norm{ \left[ \begin{array}{cc}
          \mathbf{F}^{1/2} & 0 \\
          0 & \sqrt{\frac{\lambda}{t-1}} I_{C} \end{array} \right]
      \left( \left[ \begin{array}{c}
            \bar{U} \\
            I_{C} \end{array} \right] z - \left[ \begin{array}{c}
            \epsilon \\
            0 \end{array} \right] \right) }_2^2 \\
    &\equiv \arg\min_{z \in \mathbb{R}^C} \norm{\mathbf{D}^{1/2}
      \left( \mathbf{A} z - \mathbf{b} \right) }_2^2 =
    (\mathbf{A}^\top \mathbf{D} \mathbf{A})^{-1} \mathbf{A}^\top
    \mathbf{D} \mathbf{b},
  \end{align*}
  with $\mathbf{D}^{1/2}$ being a $(A+C) \times (A+C)$ diagonal \&
  positive semidefinite matrix,
  $\mathbf{A}^\top \mathbf{D} \mathbf{A} = \sum_{a \in \mathcal{A}}
  \bar{\bf u}_{a} \bar{\bf u}_{a}^\top f_a(t-1) + \frac{\lambda}{t-1}I$
  positive definite, and $\mathbf{A}$ having full column rank $C$. A
  result of \citet[Corollary 2.3]{forsgren1996linear} can now be
  applied to yield
  \begin{align*}
    \norm{(\mathbf{A}^\top \mathbf{D} \mathbf{A})^{-1} \mathbf{A}^\top
      \mathbf{D}}_2 &\leq \max_{J} \norm{\mathbf{A}_J^{-1}}_2
  \end{align*}
  where $J$ ranges over all subsets of full-rank rows of $\mathbf{A}$,
  and $\mathbf{A}_J$ is the $C \times C$ submatrix of $\mathbf{A}$
  formed by picking rows $J$. Thus, $\norm{z_{t-1}}_2 \leq \left(
    \max_{J} \norm{\mathbf{A}_J^{-1}}_2 \right)
  \norm{\epsilon}_2$. This proves the lemma.
  \end{proof}

\begin{mylemma}[Critical radius]
  \label{lem:criticalepsilon}
  Let $\bar{\bf u}_{a^\star}^\top {\bf v}^\circ > \bar{\bf u}_a^\top {\bf v}^\circ$
  $\forall a \neq a^\star$. Then, the following are equivalent:
\beq \norm{{\bf v} - {\bf v}^\circ}_2 \leq \alpha(\bar{U})
  \norm{\epsilon}_2 \; \Rightarrow \; \bar{\bf u}_{a^\star}^\top {\bf v} -
  \bar{\bf u}_a^\top {\bf v} > \frac{\bar{\bf u}_{a^\star}^\top {\bf v}^\circ -
  \bar{\bf u}_a^\top {\bf v}^\circ}{2} \quad \forall a \neq a^\star, \label{eqn:critical1} \eeq
and
\beq \norm{\epsilon}_2 < \min_{a \neq a^\star} \; \frac{\bar{\bf u}_{a^\star}^\top
  {\bf v}^\circ - \bar{\bf u}_{a}^\top {\bf v}^\circ}{2 \alpha(\bar{U})
  \norm{\bar{\bf u}_{a^\star} - \bar{\bf u}_{a}}_2}. \label{eqn:critical2} \eeq
\end{mylemma}

\begin{proof}[Proof of Lemma \ref{lem:criticalepsilon}]
  Assuming (\ref{eqn:critical2}), observe that when ${\bf v}$ lies in
  the interior of an $\alpha(\bar{U}) \norm{\epsilon}_2$-ball around
  ${\bf v}^\circ$, we have, for any $a \neq a^\star$,
  \begin{align*}
    \left( \bar{\bf u}_{a^\star} - \bar{\bf u}_a \right )^\top {\bf v} &= \left( \bar{\bf u}_{a^\star} - \bar{\bf u}_a \right )^\top {\bf v}^\circ + \left( \bar{\bf u}_{a^\star} - \bar{\bf u}_a \right )^\top \left( {\bf v} - {\bf v}^\circ \right) \\
    &\geq \left( \bar{\bf u}_{a^\star} - \bar{\bf u}_a \right )^\top {\bf v}^\circ + \min_{\norm{\psi}_2 \leq \alpha(\bar{U})
  \norm{\epsilon}_2} \left( \bar{\bf u}_{a^\star} - \bar{\bf u}_a \right )^\top \psi \\
    &= \left( \bar{\bf u}_{a^\star} - \bar{\bf u}_a \right)^\top {\bf v}^\circ - {\alpha(\bar{U}) \norm{\epsilon}_2}{\norm{\bar{\bf u}_{a^\star} - \bar{\bf u}_a}_2} \\
    &> \left( \bar{\bf u}_{a^\star} - \bar{\bf u}_a \right)^\top {\bf v}^\circ - {\alpha(\bar{U})} {\norm{\bar{\bf u}_{a^\star} - \bar{\bf u}_a}_2}  \frac{\bar{\bf u}_{a^\star}^\top
  {\bf v}^\circ - \bar{\bf u}_{a}^\top {\bf v}^\circ}{2\alpha(\bar{U})
  \norm{\bar{\bf u}_{a^\star} - \bar{\bf u}_{a}}_2} \\
    &= \frac{\bar{\bf u}_{a^\star}^\top {\bf v}^\circ -
  \bar{\bf u}_a^\top {\bf v}^\circ}{2},
  \end{align*}
  which proves one direction of the lemma. For the other direction,
  note that if
  $\norm{\epsilon}_2 \geq \frac{\bar{\bf u}_{a^\star}^\top {\bf v}^\circ -
    \bar{\bf u}_{a}^\top {\bf v}^\circ}{2\alpha(\bar{U}) \norm{\bar{\bf u}_{a^\star} -
      \bar{\bf u}_{a}}_2}$
  for some $a \neq a^\star$, then by setting
  ${\bf v} = {\bf v}^\circ - \frac{\left(\bar{\bf u}_{a^\star}^\top {\bf v}^\circ -
      \bar{\bf u}_{a}^\top {\bf v}^\circ\right)(\bar{\bf u}_{a^\star} -
    \bar{\bf u}_{a})}{2 \norm{\bar{\bf u}_{a^\star} - \bar{\bf u}_{a}}_2^2}$, we have both
  \begin{align*}
    \norm{{\bf v} - {\bf v}^\circ}_2 &= \norm{\frac{\left(\bar{\bf u}_{a^\star}^\top {\bf v}^\circ -
                                     \bar{\bf u}_{a}^\top {\bf v}^\circ\right)(\bar{\bf u}_{a^\star} -
                                     \bar{\bf u}_{a})}{2\norm{\bar{\bf u}_{a^\star} - \bar{\bf u}_{a}}_2^2}}_2 = \frac{ \bar{\bf u}_{a^\star}^\top {\bf v}^\circ - \bar{\bf u}_{a}^\top {\bf v}^\circ }{2\norm{\bar{\bf u}_{a^\star} - \bar{\bf u}_{a}}_2} \leq \alpha(\bar{U}) \norm{\epsilon}_2
  \end{align*}
and
\begin{align*}
  \left( \bar{\bf u}_{a^\star} - \bar{\bf u}_a \right )^\top {\bf v} &= \left( \bar{\bf u}_{a^\star} - \bar{\bf u}_a \right )^\top {\bf v}^\circ - \left( \bar{\bf u}_{a^\star} - \bar{\bf u}_a \right )^\top \frac{\left(\bar{\bf u}_{a^\star}^\top {\bf v}^\circ -
      \bar{\bf u}_{a}^\top {\bf v}^\circ\right)(\bar{\bf u}_{a^\star} -
    \bar{\bf u}_{a})}{2 \norm{\bar{\bf u}_{a^\star} - \bar{\bf u}_{a}}_2^2} = \frac{\left( \bar{\bf u}_{a^\star} - \bar{\bf u}_a \right )^\top {\bf v}^\circ}{2}
\end{align*}
which contradicts (\ref{eqn:critical1}), and we are done.
\end{proof}

\subsection{Proof of Lemma~\ref{lem:pOFULSetting}}
We begin by establishing some auxiliary technical results, which
together imply Lemma \ref{lem:pOFULSetting}.

\begin{mylemma}[Controlling $\alpha_n$]
\label{lem:controlalphan}
If $n$ is large enough so that (\ref{eqn:Uestimates}) and 
\beq
  \Diamond A^{3}C
  \sqrt{\sum_{i=1}^n\gamma_i^{-2}\frac{\log(4A^3/\delta)}{2n^2}} \leq \frac{1}{2\alpha_\star}, \label{eqn:condition1}
\eeq
hold, then with probability at least $1-\delta$, 
\beq \alpha_n \leq 2 \alpha_\star.  \eeq
\end{mylemma}

\begin{proof}[Proof of Lemma \ref{lem:controlalphan}]
    The first step is to estimate the factor $\alpha$ in the analysis of
  Perturbed OFUL. Towards this, note that the quantity
  $\alpha \equiv \alpha(\bar{U})$ in our setting becomes
\[ \alpha_n \equiv \alpha_n(\bar{U}_n) = \max_{J}
\norm{(u^\diamond_n)_{J}^{-1}}_2, \] where $u^\diamond_n \bydef
\left[ \begin{array}{c}
    \bar{U}_n \\
    I_{C} \end{array} \right]$ has rank $C$, and $J$ ranges over all
combinations of its $C$ full-rank rows. For any such subset of $C$
linearly independent rows $J$, we have, after denoting $u^\diamond
\bydef \left[ \begin{array}{c}
    U \\
    I_{C} \end{array} \right]$, that
\begin{align*}
  \norm{(u^\diamond_n)_{J}^{-1}}_2 &\leq
  \norm{(u^\diamond)_{J}^{-1}}_2 + \norm{ (u^\diamond_n)_{J}^{-1} -
    (u^\diamond)_{J}^{-1} }_2.
\end{align*}
The final term above can be bounded using \citet[Lemma
E.4]{anandkumar2012method} -- a version of Theorem 2.5 in \citet{stewart1990matrix}. Assuming $(u^\diamond)_{J}$ is invertible, and
$\norm{(u^\diamond)_{J}^{-1} \left( (u^\diamond_n)_{J} -
    (u^\diamond)_{J} \right) }_2 < 1$,
then $(u^\diamond_n)_{J}$ is invertible, and a resulting bound on the
norm of its inverse lets us write
\begin{align*}
  \norm{(u^\diamond)_{J}^{-1}}_2 + \norm{ (u^\diamond_n)_{J}^{-1} -
    (u^\diamond)_{J}^{-1} }_2 &\leq \norm{(u^\diamond)_{J}^{-1}}_2 +
  \frac{\norm{(u^\diamond_n)_{J} - (u^\diamond)_{J}}_2
    \norm{(u^\diamond)_{J}^{-1}}_2^2 }{1 - \norm{(u^\diamond)_{J}^{-1}
      \left( (u^\diamond_n)_{J} - (u^\diamond)_{J} \right)}_2}.
\end{align*}
Writing $J = J_u \cup J_l$ ($u$ and $l$ stand for ``upper'' and
``lower'') with $J_l$ representing the subset of rows taken from the
bottom $C$ rows of $u^\diamond_n$ (i.e., $I_C$), we have
\[ (u^\diamond_n)_{J} - (u^\diamond)_{J} = \left[ \begin{array}{c}
    (\bar{U}_n - U)_{J_u} \\
    0 \end{array} \right].\]

Thus, with $\norm{\cdot}_F$ denoting the Frobenius norm, and using the
dominance of the Frobenius norm over the matrix $2$-norm, with
probability at least $1 - \delta$,
\begin{align}
  \norm{(u^\diamond_n)_{J} - (u^\diamond)_{J}}_2 &\leq
  \norm{(u^\diamond_n)_{J} - (u^\diamond)_{J}}_F = \norm{ (\bar{U}_n -
    U)_{J_u} }_F \leq \norm{\bar{U}_n -
    U }_F \nonumber \\
  &= \sqrt{\sum_{c \in [C]} \norm{\bar{U}_{n,c} - U_c}_2^2 } \nonumber \\
  &\leq \Diamond A^{3}C
  \sqrt{\sum_{i=1}^n\gamma_i^{-2}\frac{\log(4A^3/\delta)}{2n^2}}, \label{eqn:rtperror}
\end{align}
from the RTP error estimate (\ref{eqn:Uestimates}).

Now, letting $\alpha \equiv \alpha(U) = \max_{J}
\norm{\left(u^\diamond_{J}\right)^{-1}}_2$, the result above implies that for any
suitable $J$,
\begin{align*}
  \norm{(u^\diamond)_{J}^{-1} \left( (u^\diamond_n)_{J} -
      (u^\diamond)_{J} \right) }_2 &\leq \norm{(u^\diamond)_{J}^{-1}
  }_2 \norm{(u^\diamond_n)_{J} - (u^\diamond)_{J} }_2 \\
  &\leq \alpha
  \norm{(u^\diamond_n)_{J} - (u^\diamond)_{J} }_2 \\
  &\leq \alpha \Diamond
  A^{3}C \sqrt{\sum_{i=1}^n\gamma_i^{-2}\frac{\log(4A^3/\delta)}{2n^2}}
     \\
  &< 1/2
\end{align*}
whenever $n$ is large enough to satisfy (\ref{eqn:condition1}).

When the condition (\ref{eqn:condition1}) above holds, we get, for
any $J$ at time $n$,
\begin{align*}
  \norm{(u^\diamond_n)_{J}^{-1}}_2 &\leq
  \norm{(u^\diamond)_{J}^{-1}}_2 + \frac{\norm{(u^\diamond_n)_{J} -
      (u^\diamond)_{J}}_2 \norm{(u^\diamond)_{J}^{-1}}_2^2 }{1 -
    \norm{(u^\diamond)_{J}^{-1} \left( (u^\diamond_n)_{J} -
        (u^\diamond)_{J} \right)}_2} \\
  &\leq \alpha + 2 \alpha^2 \norm{(u^\diamond_n)_{J} -
    (u^\diamond)_{J}}_2 \\
  &\leq \alpha + 2 \alpha^2 \Diamond A^{3}C
    \sqrt{\sum_{i=1}^n\gamma_i^{-2}\frac{\log(4A^3/\delta)}{2n^2}} \quad \mbox{[by
    (\ref{eqn:rtperror})]} \\
  &\leq \alpha + 2 \alpha^2 \frac{1}{2 \alpha} = 2\alpha\,.
\end{align*}
This shows that $\alpha_n = \max_{J}
  \norm{(u^\diamond_n)_{J}^{-1}}_2 \leq 2\alpha$.

\end{proof}

\begin{mylemma}[Sufficient condition for (\ref{eqn:ass2})]
  \label{lem:ass2sufficient}
  If $n$ is large enough so that (\ref{eqn:Uestimates}),
  (\ref{eqn:condition1}) and
  \begin{align}
      \Diamond A^{3}C
  \sqrt{\sum_{i=1}^n\gamma_i^{-2}\frac{\log(4A^3/\delta)}{2n^2}} 
&\leq \min\left\{ \frac{g_b}{4 \sqrt{A} \norm{{\bf v}_b}_2}, \frac{g_b}{16 \alpha_\star \sqrt{C} u_{\max} \norm{{\bf v}_b}_2  +
               g_b } \right\} \label{eqn:condition3}
  \end{align} 
  hold, then (\ref{eqn:ass2}) is satisfied with probability at least
  $1-\delta$.
\end{mylemma}

\begin{proof}[Proof of Lemma \ref{lem:ass2sufficient}]
  The term $\norm{\epsilon}_2 = \norm{\left(U - \bar{U}_n
    \right){\bf v}_b}_2$ is bounded from above by
\begin{align}
  \norm{ U - \bar{U}_n }_2 \norm{{\bf v}_b}_2 &\leq \norm{ U - \bar{U}_n }_F
  \norm{{\bf v}_b}_2 \nonumber \\
  &\leq \sqrt{C}  \norm{{\bf v}_b}_2  \Diamond 
    A^{3}\sqrt{\sum_{i=1}^n\gamma_i^{-2}\frac{C\log(4A^3/\delta)}{2n^2}} \quad \mbox{(by
    (\ref{eqn:rtperror}))} \nonumber \\
  &\equiv \sqrt{C}  \norm{{\bf v}_b}_2 \aleph_n, \mbox{say}. \label{eqn:numcond1}
\end{align}
For any $a \neq a^\star$, 
\begin{align}
  \left( \bar{\bf u}_{n,a^\star} - \bar{\bf u}_{n,a}\right)^\top{\bf v}_b &= \left( \bf{u}_{a^\star} - \bf{u}_{a}\right)^\top{\bf v}_b + \partial_a {\bf v}_b \geq \zeta_a, \label{eqn:zetabd1}
\end{align}
with
$\partial_a^\top \bydef \left( \bar{\bf u}_{n,a^\star} - \bf{u}_{a^\star} \right)
- \left( \bar{\bf u}_{n,a} - \bf{u}_{a} \right)$,
and
$\zeta_a \bydef \inf_{\norm{\xi}_2 \leq \norm{\partial_a}_2} \left(
  \bf{u}_{a^\star} - \bf{u}_{a}\right)^\top{\bf v}_b + \xi^\top {\bf v}_b$.

Also, by
(\ref{eqn:rtperror}), we have
\begin{align*}
  \max_{a \in [A]} \norm{\bar{\bf u}_{n,a} - {\bf u}_a}_2 &\leq \sqrt{AC} \max_{c \in [C]} \norm{\bar{\bf u}_{n,c} - {\bf u}_c}_2 \\
  &\leq \sqrt{AC}  \Diamond 
    A^{3}\sqrt{\sum_{i=1}^n\gamma_i^{-2}\frac{C\log(4A^3/\delta)}{2n^2}}  \\
  &=: \aleph_n \sqrt{AC}.
\end{align*}
Thus, 
\begin{align}
  \zeta_a &\geq \inf_{\norm{\xi}_2 \leq 2 \aleph_n \sqrt{AC}}
  \left( \bf{u}_{a^\star} - \bf{u}_{a}\right)^\top{\bf v}_b + \xi^\top{\bf v}_b = \left( \bf{u}_{a^\star} - \bf{u}_{a}\right)^\top{\bf v}_b - 2 \aleph_n \sqrt{AC}
  \norm{{\bf v}_b}_2. \label{eqn:zetabd2} 
\end{align}

By (\ref{eqn:zetabd1}) and (\ref{eqn:zetabd2}), for any
$a \neq a^\star$,
\begin{align}
  \left( \bar{\bf u}_{n,a^\star} - \bar{\bf u}_{n,a}\right)^\top{\bf v}_b &\geq \left( \bf{u}_{a^\star} - \bf{u}_{a}\right)^\top{\bf v}_b - 2 \aleph_n \sqrt{AC} \norm{{\bf v}_b}_2. \label{eqn:zetabd3}
\end{align}
We also have 
\begin{align}
  \norm{\bar{\bf u}_{n,a^\star} - \bar{\bf u}_{n,a}}_2 &\leq \norm{\bf{u}_{a^\star} - \bf{u}_{a}}_2 + \norm{\bar{\bf u}_{n,a^\star} - \bf {u}_{a^\star}}_2 + \norm{\bar{\bf u}_{n,a} - \bf {u}_{a}}_2 \nonumber \\
  &\leq \norm{\bf{u}_{a^\star} - \bf{u}_{a}}_2 + 2 \aleph_n \sqrt{AC} \label{eqn:denombound1}
\end{align}
whenever (\ref{eqn:condition1}) holds. Putting (\ref{eqn:numcond1}),
(\ref{eqn:zetabd3}), (\ref{eqn:denombound1}) and the conclusion of
Lemma \ref{lem:controlalphan} together, we have that condition
(\ref{eqn:ass2}) in our case, i.e, 
\[ \norm{\epsilon}_2 \equiv \norm{\left(U - \bar{U}_n \right){\bf v}_b}_2
\leq \min_{a \neq a^\star} \frac{\left( \bar{\bf u}_{n,a^\star} -
    \bar{\bf u}_{n,a}\right)^\top{\bf v}_b}{2 \alpha_n \norm{\bar{\bf u}_{n,a^\star} -
    \bar{\bf u}_{n,a}}_2 } \]
is satisfied when
\[ \sqrt{C} \norm{{\bf v}_b}_2 \aleph_n \leq \min_{a \neq a^\star} \frac{\left(
    \bf{u}_{a^\star} - \bf{u}_{a}\right)^\top{\bf v}_b - 2 \aleph_n \sqrt{AC}
  \norm{{\bf v}_b}_2}{4 \alpha_\star \norm{\bf{u}_{a^\star} - \bf{u}_{a}}_2 +
   2 \aleph_n \sqrt{AC}}. \]
This, in turn, is satisfied if 
\begin{align*}
  2 \aleph_n \sqrt{AC}\norm{{\bf v}_b}_2 &\leq \frac{1}{2} \min_{a \neq a^\star} \left(
                                     \bf{u}_{a^\star} - \bf{u}_{a}\right)^\top{\bf v}_b = \frac{g_b}{2}, \quad \mbox{and} \\
\sqrt{C} \norm{{\bf v}_b}_2 \aleph_n &\leq \frac{g_b/2}{8 \alpha_\star \sqrt{C}u_{\max}  +
   g_b/(2\norm{{\bf v}_b}_2)} \\
  \Leftrightarrow \quad \aleph_n &\leq \frac{g_b}{16 \alpha_\star {C} u_{\max} \norm{{\bf v}_b}_2  +
   g_b \sqrt{C}}.
\end{align*}

\end{proof}

\begin{mylemma}[{Control of the distortion $\rho$ due to noisy feature
    estimates}]
  \label{lem:controlrho}
  If $n$ is large enough so that (\ref{eqn:Uestimates}),
  (\ref{eqn:condition1}) and (\ref{eqn:condition3}) hold, then
  $\rho' \leq 2$ with probability at least $1 - \delta$.

\end{mylemma}

\begin{proof}[Proof of Lemma \ref{lem:controlrho}]
  We begin by considering
\begin{align*}
  \max_{a \neq a^\star} \; \frac{\left({\bf u}_{a^\star} - {\bf u}_a \right)^\top
    {\bf v}_b}{\left( \bar{{\bf u}}_{n,a^\star} - \bar{{\bf u}}_{n,a}\right)^\top{\bf v}_b } &\leq
  \max_{a \neq a^\star} \; \frac{\left({\bf u}_{a^\star} - {\bf u}_a \right)^\top
    {\bf v}_b}{ \left( {\bf u}_{a^\star} - {\bf u}_{a}\right)^\top{\bf v}_b + \partial_a {\bf v}_b}
  \leq \max_{a \neq a^\star} \; \frac{\left({\bf u}_{a^\star} - {\bf u}_a \right)^\top
    {\bf v}_b}{\zeta_a},
\end{align*}
with $\partial_a^\top \bydef \left( \bar{{\bf u}}_{n,a^\star}
    - {{\bf u}}_{a^\star} \right) - \left( \bar{{\bf u}}_{n,a}
    - {{\bf u}}_{a} \right)$, and 
  \begin{align*}
    \zeta_a \bydef \inf_{\norm{\xi}_2 \leq \norm{\partial_a}_2}
    \left( {\bf u}_{a^\star} - {\bf u}_{a}\right)^\top{\bf v}_b + \xi^\top{\bf v}_b
  \end{align*}
  as in the proof of Lemma \ref{lem:ass2sufficient}. Also, by
  (\ref{eqn:rtperror}), we have that with probability at least
  $1 - \delta$,
\begin{align*}
  \max_{a \in [A]} \norm{\bar{{\bf u}}_{n,a} - {\bf u}_a}_2 &\leq \sqrt{AC} \max_{c \in [C]} \norm{\bar{{\bf u}}_{n,c} - {\bf u}_c}_2 \\
  &\leq \sqrt{AC} \Diamond 
    A^{3}\sqrt{\sum_{i=1}^n\gamma_i^{-2}\frac{C\log(4A^3/\delta)}{2n^2}} \\
  &=: \aleph_n \sqrt{AC}, \quad \mbox{say}.
\end{align*}

Thus, 
\begin{align*}
  \zeta_a &\geq \inf_{\norm{\xi}_2 \leq 2 \aleph_n \sqrt{AC}}
  \left( {\bf u}_{a^\star} - {\bf u}_{a}\right)^\top{\bf v}_b + \xi^\top{\bf v}_b \\
  &= \left( {\bf u}_{a^\star} - {\bf u}_{a}\right)^\top{\bf v}_b - 2 \aleph_n \sqrt{AC}
  \norm{{\bf v}_b}_2 \\
  \Rightarrow \; \frac{\zeta_a}{\left( {\bf u}_{a^\star} -
      {\bf u}_{a}\right)^\top{\bf v}_b} &\geq 1 - \frac{2 \aleph_n \sqrt{AC}
    \norm{{\bf v}_b}_2}{\left( {\bf u}_{a^\star} - {\bf u}_{a}\right)^\top{\bf v}_b} \geq 1 -
  \frac{2 \aleph_n \sqrt{AC} \norm{{\bf v}_b}_2}{g_b},
\end{align*}
where $g_b \bydef \min_{a \neq a^\star} \left( {\bf u}_{a^\star} - {\bf u}_{a}
\right)^\top{\bf v}_b > 0$ is the minimum gap for user $b$ across suboptimal
actions.  

Provided that (\ref{eqn:Uestimates}), (\ref{eqn:condition1}) and
(\ref{eqn:condition3}) hold, we get that with probability at least
$1 -\delta$,
$\frac{\zeta_a}{\left( {\bf u}_{a^\star} - {\bf u}_{a}\right)^\top{\bf v}_b} \geq
\frac{1}{2}$
for each $a \neq a^\star$. Also, by the definition of $a^\star$, the
denominator is positive, i.e.,
$\left( {\bf u}_{a^\star} - {\bf u}_{a} \right)^\top{\bf v}_b > 0$. Hence,
\[ \max_{a \neq a^\star} \; \frac{\left({\bf u}_{a^\star} - {\bf u}_a \right)^\top
    {\bf v}_b}{\left( \bar{{\bf u}}_{n,a^\star} - \bar{{\bf u}}_{n,a}\right)^\top{\bf v}_b } \leq 2, \]
completing the proof of the result.
\end{proof}

\begin{mylemma}[Bounding $R_\cX$]
  \label{lem:RXbound}
  If $n$ is large enough so that (\ref{eqn:Uestimates}) and
  (\ref{eqn:condition1}) hold, then
  \[ R_\cX \leq \frac{\sqrt{A}}{2\alpha_\star} + \max_{a \in
    \mathcal{A}} \norm{{\bf u}_a}_2,\]
  with probability at least $1 - \delta$.
\end{mylemma}

\begin{proof}[Proof of Lemma \ref{lem:RXbound}]
  Conditions (\ref{eqn:Uestimates}) and (\ref{eqn:condition1}),
  together with the estimate (\ref{eqn:denombound1}), imply that for
  any action $a$,
  \begin{align*}
    \norm{\bar{{\bf u}}_{n,a}}_2 &\leq \norm{{{\bf u}}_{a}}_2 + \norm{\bar{{\bf u}}_{n,a} - {{\bf u}}_{a}}_2 \leq \norm{{{\bf u}}_{a}}_2 + \aleph_n \sqrt{AC} \leq \norm{{{\bf u}}_{a}}_2 + \sqrt{A}/(2\alpha_\star).
  \end{align*}
  with probability at least $1-\delta$.
\end{proof}

In order to conclude the proof of Lemma~\ref{lem:pOFULSetting},
we gather the conditions from Lemma~\ref{lem:controlalphan} and Lemma~\ref{lem:ass2sufficient}. After some simplifications,
both conditions are satisfied as soon as
\beqan
\frac{n^2}{\sum_{i=1}^n\gamma_i^{-2}} \geq 
\Diamond^2 A^6C^2\log(4A^3/\delta)\max\bigg\{
2\alpha_\star^2,
\frac{8A||{\bf v}_b||_2^2}{g_b^2},
\frac{2^7\alpha_\star^2 Cu_{\max}^2 ||{\bf v}_b||_2^2}{g_b^2}
+1/2
\bigg\}\,.
\eeqan

\section{Proof of Theorem~\ref{thm:main}}

\begin{proof}
Let $n_0$ be the first mini-session
such that both conditions in Lemma~\ref{lem:pOFULSetting} are satisfied, that is such that
\beqan
\frac{n_0}{\sum_{i=1}^{n_0} \gamma_i^{-2}} \geq \varhexagon_\delta\,.
\eeqan

The cumulative regret $\kR_T = \sum_{t=1}^T r_t$
of Algorithm~\ref{alg:PerOFULwithExploration}
satisfies
\beqan
\kR_T &=& \sum_{n=1}^N \sum_{l=1}^\ell r_{n,l} \\
&\leq & (n_0-1)\ell + \sum_{b\in[B]} \sum_{n=n_0}^N \sum_{l=1}^\ell r_{n,l}\indic{b_n=b}
\eeqan
where $r_t \equiv r_{n,l} \eqdef {\bf u}_{a^\star_{b_n}}^\top {\bf v}_{b_n} - {\bf u}_{a_{n,l}}^\top {\bf v}_{b_n}$ is the instantaneous regret 
of Algorithm~\ref{alg:PerOFULwithExploration} at time $t=\ell n+k$ when the current user is $b_n=b$.
Using the notations of Algorithm~\ref{alg:PerOFULwithExploration}, 
it holds that
\beqan
\Esp[r_t|b_n=b] &=& \Esp[r_t\indic{p_n=1}|b_n=b] + \Esp[r_t\indic{p_n=0}|b_n=b] \\
&\leq & \Esp[ {\bf u}_{a^\star_b}^\top {\bf v}_b - {\bf u}_{\tilde a_{n,k}}^\top {\bf v}_b](1-\gamma_n) + \gamma_n.\\
&\leq& \Esp[ {\bf u}_{a^\star_b}^\top {\bf v}_b - {\bf u}_{\tilde a_{n,k}}^\top {\bf u}_b] + \gamma_n,
\eeqan
where $\tilde a_{n,k}$ is an action output by 
an instance of \OFUL\ for user $b_n=b$.
Thus, we have
\begin{align}
  &\Esp[\kR_T | b_1,\dots,b_N]\nonumber\\
   &\leq (n_0-1)\ell + \Esp\Bigg[\sum_{b\in[B]} \sum_{n=n_0}^N \sum_{l=1}^\ell  \Big({\bf u}_{a^\star_b}^\top {\bf v}_b - {\bf u}_{\tilde a_{n,l}}^\top {\bf v}_b\Big)\indic{b_n=b} \bigg| b_1,\dots,b_N\Bigg]
                                +\ell\sum_{n=n_0}^N\gamma_n\,. \nonumber \\
                              &= (n_0-1)\ell + \sum_{b\in[B]} \Esp \Bigg[\underbrace{\sum_{\substack{n_0 \leq n \leq N,\\b_n = b}} \; \sum_{l=1}^\ell  \Big({\bf u}_{a^\star_b}^\top {\bf v}_b - {\bf u}_{\tilde a_{n,l}}^\top {\bf v}_b\Big)}_{(\star)} \bigg| b_1,\dots,b_N\Bigg] + \ell\sum_{n=n_0}^N\gamma_n. \label{eqn:RegDecompo}
\end{align}
For each user $b \in [B]$, the expectation in the right-hand side
above corresponds to the cumulative regret of the \OFUL\ strategy when
interacting with user $b$ in mini-sessions $n_0$ through $N$, and when
given at each mini-session $n$ the set of perturbed feature vectors
$\ol{U}_n$.  Let $N_{b,n_0} = \sum_{n=n_0}^N \indic{b_n=b}$ count the
total number of mini-sessions from $n_0$ in which user $b$ is present
(note that $\sum_{b\in[B]} N_{b,1} = N$ and
$\sum_{b\in[B]} \ell N_{b,1} = T$). Let us denote the term $(\star)$
in the above explicitly using
$ \kR_{b,N_{b,n_0}}(\{\ol{U}_n\}_{n\in[n_0,N],b_n=b})$.

We can now use the \OFUL\ robustness guarantee -- a natural technical
extension\footnote{Although Theorem~\ref{thm:pOFUL2} holds only for a
  fixed perturbation $\epsilon$ and feature set $\bar{\bf u}$, it is
  not hard to see that a modification of it, with time-varying
  $\epsilon_t$, $\bar{\bf u}_t$ and $\rho'$ being the largest
  $\rho_t'$ over all times $t$, yields the same conclusion (regret
  bound). We provide this extension in Theorem~\ref{thm:pOFUL2ext} in
  Appendix~\ref{sec:TimeVar} below.}  of Theorem~\ref{thm:pOFUL2}
along with Lemma~\ref{lem:pOFULSetting} -- to obtain that, for a given
user sequence $b_1, \ldots, b_N$, with probability at
least\footnote{Although the time horizons played by each OFUL instance
  per user, $N_{b,n_0}$, are technically random and unknown to the
  instance at the start, conditioning on the sequence of users
  arriving at each time instant lets us use the conclusion of
  Lemma~\ref{lem:pOFULSetting}.}
$1 - 2\delta - \delta = 1 - 3 \delta$,
\beqan
\lefteqn{\kR_{b,N_{b,n_0}}(\{\ol{U}_n\}_{n\in[n_0,N],b_n=b}) \leq}\\
&& 16
\sqrt{\ell N_{b,n_0} \; C \log\left(1 +
    \frac{\ell N_{b,n_0}R_\cX^2}{\lambda C}\right)}
\left(\lambda^{1/2}R_\Theta + R\sqrt{2 \log\frac{1}{\delta} + C \log
    \left(1+\frac{\ell N_{b,n_0} R_\cX^2}{\lambda C}\right)} \right). 
\eeqan
This in turn implies that
\begin{align*} 
  &\sum_{b \in B} \Esp\left[ \kR_{b,N_{b,n_0}}(\{\ol{U}_n\}_{n\in[n_0,N],b_n=b}) \bigg| b_1,\dots,b_N \right] \\
  &\stackrel{(a)}{\leq} 16 \sum_{b \in B}
    \sqrt{\ell N_{b,n_0} \; C \log\left(1 +
    \frac{\ell N_{b,n_0}R_\cX^2}{\lambda C}\right)}
    \left(\lambda^{1/2}R_\Theta + R\sqrt{2 \log\frac{1}{\delta} + C \log
    \left(1+\frac{\ell N_{b,n_0} R_\cX^2}{\lambda C}\right)} \right)\\
    &\quad\quad + \sum_{b \in B} 3\delta \ell N_{b,n_0} \\
  &\stackrel{(b)}{\leq} 16 \sum_{b \in B}
    \sqrt{\ell N_{b,n_0} \; C \log\left(1 +
    \frac{\ell N_{b,n_0}R_\cX^2}{\lambda C}\right)}
    \left(\lambda^{1/2}R_\Theta + R\sqrt{2 \log\frac{1}{\delta} + C \log
    \left(1+\frac{\ell N_{b,n_0} R_\cX^2}{\lambda C}\right)} \right)\\
    &\quad\quad + 3\delta T. 
\end{align*}
The last term on the right-hand side in $(a)$ is due to the fact that
with probability at most $3 \delta$, the per-user regret
$\kR_{b,N_{b,n_0}}(\{\ol{U}_n\}_{n\in[n_0,N],b_n=b})$ can be as large
as $\ell N_{b,n_0}$ (the total number of time slots for which user $b$
interacts with the system). The corresponding term in $(b)$ is by
using $\sum_{b\in[B]} \ell N_{b,1} = T$. Further bounding using the
Cauchy-Schwarz inequality
$\sum_{b \in B} \sqrt{\ell N_{b,n_0}} \leq \sqrt{BT}$ gives
\begin{align*}
  &\sum_{b \in B} \Esp\left[ \kR_{b,N_{b,n_0}}(\{\ol{U}_n\}_{n\in[n_0,N],b_n=b}) \bigg| b_1,\dots,b_N \right] \\
  &\leq 16 \sum_{b \in B}
    \sqrt{\ell N_{b,n_0} \; C \log\left(1 +
    \frac{TR_\cX^2}{\lambda C}\right)}
    \left(\lambda^{1/2}R_\Theta + R\sqrt{2 \log\frac{1}{\delta} + C \log
    \left(1+\frac{T R_\cX^2}{\lambda C}\right)} \right) + 3\delta T \\
  &\leq 16
    \sqrt{BT C \log\left(1 +
    \frac{TR_\cX^2}{\lambda C}\right)}
    \left(\lambda^{1/2}R_\Theta + R\sqrt{2 \log\frac{1}{\delta} + C \log
    \left(1+\frac{T R_\cX^2}{\lambda C}\right)} \right) + 3\delta T.
\end{align*}

Plugging this estimate into \eqref{eqn:RegDecompo}, we obtain that
\begin{align*}
  &\Esp[R_T|b_1,\dots b_N]  \leq \ell \left(n_0 - 1 + \sum_{n=n_0}^N\gamma_n\right)\\
  & + 16
    \sqrt{BT C \log\left(1 +
    \frac{TR_\cX^2}{\lambda C}\right)}
    \left(\lambda^{1/2}R_\Theta + R\sqrt{2 \log\frac{1}{\delta} + C \log
    \left(1+\frac{T R_\cX^2}{\lambda C}\right)} \right) + 3\delta T.
\end{align*}

{\bf Expliciting $n_0$ and tuning $\gamma_n$}
The next step is to control the term
$n_0 - 1 + \sum_{n=n_0}^N \gamma_n$.
To this end, we explicit $n_0$ and optimize $\gamma_n$.
We write $\varhexagon\equiv\varhexagon_\delta$ in the sequel for convenience.

If $\gamma_n = \min\{1,\varhexagon^{1/2}n^{-1/2}\}$, then
\beqan
\frac{n^2}{\sum_{m=1}^{n} \gamma_m^{-2}}
 &=&   \frac{n^2}{\lceil \varhexagon \rceil + \frac{1}{\varhexagon}\sum_{m>\lceil \varhexagon \rceil}^{n} m}\\
 &=&\frac{2\varhexagon n^2}{\lceil \varhexagon \rceil 2 \varhexagon
 + n(n+1) - \lceil \varhexagon \rceil(\lceil \varhexagon \rceil-1)}\\
&\geq& 
 \frac{2\varhexagon}{1+1/n+(\lceil \varhexagon\rceil \varhexagon)/n^2}\,.
\eeqan
Thus, this is higher than $\varhexagon$ if 
$n^2-n - \lceil \varhexagon\rceil \varhexagon \geq 0$, that is if
$n \geq n_0 \eqdef \lceil 
1/2 + \sqrt{\lceil \varhexagon\rceil \varhexagon +1/4}
\rceil $. Since $n_0\geq\lceil \varhexagon\rceil$, we immediately get
%

\beqan
\sum_{n=n_0}^N \gamma_n &\leq&  
\varhexagon^{1/2}n_0^{-\frac{1}{2}}
+  2\varhexagon^{1/2}\Big(N^{\frac{1}{2}}- n_0^{\frac{1}{2}}\Big)\\
&\leq& 1+2\varhexagon^{1/2}\Big(N^{\frac{1}{2}}- n_0^{\frac{1}{2}}\Big)\,.
\eeqan
Thus, we obtain
\beqan
n_0-1 + \sum_{n=n_0}^N \gamma_n &\leq& 2\varhexagon^{1/2}N^{\frac{1}{2}}
+ n_0 - 2\varhexagon^{1/2}n_0^{1/2}\\
&\leq& 2\varhexagon^{1/2}N^{\frac{1}{2}} + 
n_0 - 2\sqrt{\varhexagon\lceil \varhexagon\rceil}
\eeqan
Using the fact that $\varhexagon>1$, the bound simplifies to
\beqan
n_0-1 + \sum_{n=n_0}^N \gamma_n
\leq 2\sqrt{\varhexagon N} + 1\,.
\eeqan

If, on the other hand, a bound on $\varhexagon$ is not readily available
beforehand, then choosing $\gamma_n = \sqrt{\log(1+n)/n}$,
$n \geq 1$, gives, via a crude bound,
\begin{align*}
  \sum_{m=1}^n \gamma_m^{-2} &= \sum_{m=1}^n m / \log(1+m) \leq \sum_{m=1}^{\sqrt{n}} m/\log 2 + \sum_{m = \sqrt{n}}^n m/\log(1 + \sqrt{n}) \\
                             & \leq n/\log 2 + n^2/\log\sqrt{n} \leq 2n^2/\log\sqrt{n}\\
  \Rightarrow \quad \frac{n_0^2}{\sum_{m=1}^{n_0} \gamma_m^{-2}} &\geq \frac{n_0^2}{2n_0^2/\log\sqrt{n_0}} = \frac{\log n_0}{4}. 
\end{align*}
The bound above is at least $\varhexagon$ provided
$n_0 \geq \exp(4 \varhexagon)$. Thus, we finally get that, upon setting $\delta = 1/\sqrt{T}$, the total
expected regret satisfies (as an order-wise function of $T$)
\begin{align*}
  &\Esp[R_T] \\
  &\leq \ell \left( \exp(4\varhexagon) + \sum_{n=1}^N \sqrt{\log(n+1)/n} \right) \\
  & \quad \quad + 16
    \sqrt{BT C \log\left(1 +
    \frac{TR_\cX^2}{\lambda C}\right)}
    \left(\lambda^{1/2}R_\Theta + R\sqrt{\log T + C \log
    \left(1+\frac{T R_\cX^2}{\lambda C}\right)} \right) + 3\sqrt{T} \\
  &= O\left( C \sqrt{BT} \log T \right)\,.
\end{align*}

\end{proof}

\section{Extension of Theorem~\ref{thm:pOFUL2}: Robustness of OFUL's
  regret with time-varying features}
\label{sec:TimeVar}
We now control the robust regret for user $b$
$\kR_{b,N_{b,n_0}}(\{\ol{U}_n\}_{n\in[n_0,N],b_n=b})=\sum_{\substack{n_0 \leq n \leq N,\\b_n = b}} \; \sum_{l=1}^\ell  \Big({\bf u}_{a^\star_b}^\top {\bf v}_b - {\bf u}_{\tilde a_{n,l}}^\top {\bf v}_b\Big)$,
when OFUL is run with evolving feature matrices
$\{\ol{U}_n\}_{n\in[n_0,N],b_n=b}$ with decreasing feature error
$\epsilon_n=(U-\ol{U}_n){\bf v}_b$, instead of a fixed 
$\ol{U}$ with fixed error $\epsilon=(U-\ol{U}){\bf v}_b$.

We reindex the ${n\in[n_0,N],b_n=b}$ as $t=1,\dots,..$
and prove the following result.

\begin{mytheorem}[OFUL robustness result, extension of
  Theorem~\ref{thm:pOFUL2} for time-varying features]
  \label{thm:pOFUL2ext}
  Assume $||{\bf v}^\circ||_2\leq R_\Theta$,
  $\lambda \geq \max \left\{1,R_\cX^2,1/{4 R_{\Theta}^2} \right\}$,
  $\forall a\in\cA$, $t\leq T$, $||\bar{\bf u}_a^{(t)}||_2 \leq R_\cX$
  and $|m_a| \leq 1$, and that for all $t\leq T$,
  $\arg \max_{a \in \cA} \bar{\bf u}_{a}^{{(t)}\top} {\bf v}^\circ =
  \{a^\star\}$
  (i.e., the linearly realizable approximation with respect to the
  current features has $a^\star$ as its unique optimal action). If
  \beq \norm{\epsilon^{(t)}}_2 \equiv \norm{{\bf m} - \bar{U}^{(t)}
    {\bf v}^\circ}_2 < \min_{a \neq a^\star} \; \frac{\bar{\bf
      u}_{a^\star}^{{(t)}\top} {\bf v}^\circ - \bar{\bf
      u}_{a}^{{(t)}\top} {\bf v}^\circ}{2\alpha(\bar{U}^{{(t)}\top})
    \norm{\bar{\bf u}_{a^\star}^{(t)} - \bar{\bf
        u}_{a}^{(t)}}_2}, \label{eqn:ass2ext} \eeq then with probability
  at least $1 - \delta$, for all $T \geq 0$,
  \begin{align*}
    \kR_T &\leq 8 \rho' \sqrt{TC \log\left(1 +
      \frac{TR_\cX^2}{\lambda C}\right)}\left(\lambda^{1/2}R_\Theta +
    R\sqrt{2 \log\frac{1}{\delta} + C \log
      \left(1+\frac{TR_\cX^2}{\lambda C}\right)} \right), 
  \end{align*}
 where
    $\rho' \bydef \max_{t}\max\left\{1,  \max_{a \neq a^\star} \; \frac{m_{a^\star} -
      m_a}{\bar{\bf u}_{a^\star}^{{(t)}\top} {\bf v}^\circ - \bar{\bf u}_{a}^{{(t)}\top} {\bf v}^\circ } \right\}$. 
\end{mytheorem}
\begin{proof}
  Let $\mathbf{M}_{1:t} = ({\bf m}_{A_1}, \ldots, {\bf m}_{A_t})^\top$. The
  argument used to prove Theorem 2 in Yadkori et al, 2011, shows that
  \begin{align*} \hat{\bf v}_{t-1} &= V_{t-1}^{-1}
    \bar{\bf U}_{1:t-1}^{(t)}\eta_{1:t-1} + V_{t-1}^{-1}
    \bar{\bf U}_{1:t-1}^{(t)} \mathbf{M}_{1:t-1} 
  \end{align*}
  where $\eta_{1:t-1} \bydef (\eta_1, \ldots, \eta_{t-1})$ is the
  observed noise sequence, and where $\bar{\bf U}_{1:t-1}^{(t)}$ is
  the matrix built from the time varying features at time $t$ and the
  action sequence thus far. Let
  $\mathbf{E}_{1:t-1}^{(t)} \bydef (\epsilon_{A_1}^{(t)}, \ldots,
  \epsilon_{A_t}^{(t)})^\top = \mathbf{M}_{1:t-1} - \bar{\bf U}_{1:t-1}^{(t)} {\bf
    v}^\circ$. We then have
  \begin{align*}
    \hat{\bf v}_{t-1} &= V_{t-1}^{-1} \bar{\bf U}_{1:t-1}^{(t)}
    \eta_{1:t-1} + V_{t-1}^{-1}
    \bar{\bf U}_{1:t-1}^{(t)} \mathbf{M}_{1:t-1} \\
    &= V_{t-1}^{-1} \bar{\bf U}_{1:t-1}^{(t)} \eta_{1:t-1} +
    V_{t-1}^{-1} \bar{\bf U}_{1:t-1}^{(t)} \left( \bar{\bf U}_{1:t-1}^{{(t)}\top}
      {\bf v}^\circ + \mathbf{E}_{1:t-1}^{(t)}\right) \\
    &= V_{t-1}^{-1} \bar{\bf U}_{1:t-1}^{(t)} \eta_{1:t-1} +
    {\bf v}^\circ - \lambda V_{t-1}^{-1} {\bf v}^\circ + V_{t-1}^{-1}
    \bar{\bf U}_{1:t-1}^{(t)} \mathbf{E}_{1:t-1}^{(t)}.
  \end{align*}

  Thus, letting ${\bf v}_{t-1}^{+} \bydef {\bf v}^\circ +
  V_{t-1}^{-1}\bar{U}_{1:t-1}^{(t)} \mathbf{E}_{1:t-1}^{(t)}$, and using the
  above with techniques from Yadkori et al together with
  $\norm{{\bf v}^\circ}_2 \leq R_\Theta$, we have that
  \begin{align*}
    {\bf v}_{t-1}^{+} \in \cC_{t-1}
  \end{align*}
  with probability at least $1-\delta$.

  Now, let $a_{t-1}^{+} \in \arg\max_{a \in \mathcal{A}} \bar{\bf u}_{a}^{{(t)}\top}
  {\bf v}_{t-1}^+$ be an optimal action corresponding to the
  approximate parameter ${\bf v}_{t-1}^+$ and approximate feature $\bar{\bf u}_{a}^{{(t)}\top}$, and define the instantaneous
  regret at time $t$ {\em with respect to the approximate parameter}
  as
  \[ r_t^+ \bydef \bar{\bf u}_{a_{t-1}^{+}}^{{(t)}\top} {\bf
    v}_{t-1}^+ - \bar{\bf u}_{A_t}^{{(t)}\top} {\bf v}_{t-1}^+ \geq
  0.\]
  We now bound this approximate regret using arguments along the lines
  of Yadkori et al, 2011 as follows. Write
  \begin{align}
    r_t^+ &= \bar{\bf u}_{a_{t-1}^{+}}^{{(t)}\top} {\bf v}_{t-1}^+ - \bar{\bf u}_{A_t}^{{(t)}\top}
            {\bf v}_{t-1}^+ \nonumber \\
          &\leq \bar{\bf u}_{A_t}^{{(t)}\top} \tilde{{\bf v}}_t - \bar{\bf u}_{A_t}^{{(t)}\top} {\bf v}_{t-1}^+
            \quad \quad \mbox{(since $(A_t, \tilde{{\bf v}}_t)$ is optimistic)}  \nonumber \\
          &= \bar{\bf u}_{A_t}^{{(t)}\top} \left( \tilde{{\bf v}}_t - {\bf v}_{t-1}^+ \right)  \nonumber \\
          &= \bar{\bf u}_{A_t}^{{(t)}\top} \left( \tilde{{\bf v}}_t - \hat{{\bf v}}_{t-1}
            \right) + \bar{\bf u}_{A_t}^{{(t)}\top} \left( \hat{{\bf v}}_{t-1} - {\bf v}_{t-1}^+
            \right)  \nonumber \\
          &\leq \norm{\bar{\bf u}_{A_t}^{(t)}}_{V_{t-1}^{-1}} \norm{ \tilde{{\bf v}}_t - \hat{{\bf v}}_{t-1}
            }_{V_{t-1}}  + \norm{\bar{\bf u}_{A_t}^{(t)}}_{V_{t-1}^{-1}} \norm{ \hat{{\bf v}}_{t-1} - {\bf v}_{t-1}^+}_{V_{t-1}} \quad \quad
            \mbox{(Cauchy-Schwarz's inequality)}  \nonumber \\
          &\leq 2 {D_{t-1}} \norm{\bar{\bf u}_{A_t}^{(t)}}_{V_{t-1}^{-1}}. \label{eqn:rtplusext}
  \end{align}

  Noting that $m_a \in [-1,1]$ $\forall a$, the regret can be
  written as
  \begin{align*}
    R_T &= \sum_{t=1}^T \left( m_{a^\star} - m_{A_t} \right)
          = \sum_{t=1}^T \min\{ m_{a^\star} - m_{A_t}, 2 \} \\
        &= \rho' \sum_{a \neq a^\star} \sum_{t=1}^T \min \left\{ \frac{m_{a^\star} - m_{a}}{\rho'}, \frac{2}{\rho'} \right\} \indic{A_t = a}  \\
        &\leq \rho' \sum_{a \neq a^\star} \sum_{t=1}^T \min \left\{ \bar{\bf u}_{a^\star}^{{(t)}\top}
          {\bf v}^\circ - \bar{\bf u}_{a}^{{(t)}\top} {\bf v}^\circ , \frac{2}{\rho'} \right\}  \indic{A_t = a} \quad \mbox{(using the definition of $\rho'$)} \\
        &\stackrel{(a)}{\leq} \rho' \sum_{t=1}^T \min \left\{ 2\left( \bar{\bf u}_{a^\star}^{{(t)}\top}
          {\bf v}_{t-1}^+ - \bar{\bf u}_{A_t}^{{(t)}\top} {\bf v}_{t-1}^+ \right), \frac{2}{\rho'} \right\} \stackrel{(b)}{=} 2 \rho' \sum_{t=1}^T \min \left\{ \bar{\bf u}_{a_{t-1}^+}^{{(t)}\top}
          {\bf v}_{t-1}^+ - \bar{\bf u}_{A_t}^{{(t)}\top} {\bf v}_{t-1}^+ , \frac{1}{\rho'} \right\} \\
        &= 2 \rho' \sum_{t=1}^T \min \left\{ r_t^+ , \frac{1}{\rho'} \right\} = \rho' \sum_{t=1}^T \frac{2}{\rho'}
          \min \left\{ {\rho' r_t^+} , 1 \right\} \stackrel{(c)}{\leq}  \rho' \sum_{t=1}^T \frac{2}{\rho'}
          \min \left\{ 2 \rho' {D_{t-1}} \norm{\bar{\bf u}_{A_t}^{(t)}}_{V_{t-1}^{-1}} , 1 \right\} \\
        &\stackrel{(d)}{\leq} \rho' \sum_{t=1}^T {4} D_{t-1}
          \min \left\{\norm{\bar{\bf u}_{A_t}^{(t)}}_{V_{t-1}^{-1}} , 1 \right\} \\
        &\leq  \rho' \sqrt{ T \sum_{t=1}^T 16
          {D_{T}}^2  \min \left\{ \norm{\bar{\bf u}_{A_t}^{(t)}}_{V_{t-1}^{-1}}^2 , 1 \right\} } \quad \mbox{(by using Cauchy-Schwarz's inequality)}.
  \end{align*}
  In the derivation above, 
  \begin{itemize}
  \item Steps $(a)$ and $(b)$ hold because of the following. By Lemma
    \ref{lem:parameterbiasTimeVar} (to follow below),
    $\norm{{\bf v}_{t-1}^+ - {\bf v}^\circ}_2 =
    \norm{V_{t-1}^{-1}\bar{\bf U}_{1:t-1}^{(t)}
      \mathbf{E}_{1:t-1}^{(t)}}_2 \leq \alpha(\bar U_t)
    \norm{\epsilon^{(t)}}_2$.
    Since
    $\arg \max_{a \in \cA} \; \bar{\bf u}_{a}^{{(t)}\top} {\bf
      v}^\circ$
    is uniquely $a^\star$ by hypothesis, we have, thanks to Lemma
    \ref{lem:criticalepsilon}, that
    $\bar{\bf u}_{a^\star}^{{(t)}\top} {\bf v}_{t-1}^+ - \bar{\bf
      u}_{a}^{{(t)}\top} {\bf v}_{t-1}^+ > \frac{\bar{\bf
        u}_{a^\star}^{{(t)}\top} {\bf v}^\circ - \bar{\bf
        u}_{a}^{{(t)}\top} {\bf v}^\circ}{2} > 0$
    $\forall a \neq a^\star$, establishing $(a)$. This in turn shows
    that the optimal action for ${\bf v}_{t-1}^+$ is uniquely
    $a^\star$ at all times $t$, i.e.,
    $a_{t-1}^{+} = \arg\max_{a \in \mathcal{A}} \bar{\bf u}_{a}^{{(t)}\top}
    {\bf v}_{t-1}^+ = a^\star$, which is precisely equality $(b)$.

  \item {\em Remark.} In the above, Lemma~ \ref{lem:criticalepsilon} is written for
    generic $\bar{\bf u}_{a}$, $\epsilon$, so in particular applies to
    each time varying $\bar{\bf u}_{a}^{(t)}$, $\epsilon^{(t)}$. We
    also used an extended version of Lemma~ \ref{lem:parameterbias} to
    the case of varying $\bar{\bf u}_{a}^{(t)}$, $\epsilon^{(t)}$,
    which we state and prove below as Lemma~
    \ref{lem:parameterbiasTimeVar}.

  \item Inequality $(c)$ holds by (\ref{eqn:rtplusext}) and
  $(d)$ holds because $\rho' \geq 1$ by definition,
    and $D_{t-1} \geq \lambda^{1/2}R_{\Theta} \geq 1/2$ by hypothesis,
    implying that $2 \rho' D_{t-1} \geq 1$.
  \end{itemize}

  The argument from here can be continued in the same way as in
  \citet[proof of Theorem 3]{AbbPalCsa11:linbandits} to yield
  \begin{align*}
    R_T \leq 8 \rho' \sqrt{TC \log\left(1 +
    \frac{TR_\cX^2}{\lambda C}\right)} \left(\lambda^{1/2}R_\Theta +
    R\sqrt{2 \log\frac{1}{\delta} + C \log
    \left(1+\frac{TR_\cX^2}{\lambda C}\right)} \right).
  \end{align*}

This proves the theorem.

\end{proof}

\begin{mylemma}[Extension of Lemma~\ref{lem:parameterbias} to
  time-varying feature sets]
    \label{lem:parameterbiasTimeVar}
    Let
    $\epsilon_a^{(t)} = m_a - \bar{\bf u}_a^{{(t)}\top} {\bf v}^\circ$
    be the bias in arm $a$'s reward due to model error, with respect
    to the features ${\bar U}_t$, and let
    $\epsilon^{(t)} \equiv \left(\epsilon_a^{(t)}\right)_{a \in
      \mathcal{A}}$. Then, we have
    \[ \norm{V_{t-1}^{-1}\bar{\bf U}_{1:t-1}^{(t)}
      \mathbf{E}_{1:t-1}^{(t)}}_2 \leq \left( \max_{J}
      \norm{\mathbf{A}_J^{{(t)}-1}}_2 \right) \norm{\epsilon^{(t)}}_2, \]
    where $\mathbf{A}_{(A+C) \times C}^{(t)} = \left[ \begin{array}{c}
                                                        \bar{U}^{(t)} \\
                                                        I_{d} \end{array}
  \right]$,
  $\mathbf{A}_J^{(t)}$ is the $C \times C$ submatrix of
  $\mathbf{A}^{(t)}$ consisting of rows in $J$, and $J$ ranges over
  all subsets of full-rank rows of $\mathbf{A}^{(t)}$.
  \end{mylemma}

  \begin{proof}[Proof of Lemma \ref{lem:parameterbiasTimeVar}]
    Let
    $z_{t-1}^{(t)} \bydef V_{t-1}^{-1}\bar{\bf U}_{1:t-1}^{(t)}
    \mathbf{E}_{1:t-1}^{(t)} = {\bf v}_{t-1}^{+} - {\bf v}^\circ
    \in \mathbb{R}^C$,
    thus
    $\norm{\mathbf{E}_{1:t-1}^{(t)}}_\infty \leq
    \norm{\epsilon^{(t)}}_\infty = \norm{{\bf m} - \bar{U}^{(t)} {\bf
        v}^\circ}_\infty$. We now write
  \begin{align*}
    z_{t-1}^{(t)} &= \left( \sum_{s=1}^{t-1} \bar{\bf u}_{A_s}^{(t)} \bar{\bf u}_{A_s}^{{(t)}\top} + \lambda
      I \right)^{-1} \sum_{s=1}^{t-1} \epsilon_{A_s}^{(t)} \bar{\bf u}_{A_s}^{(t)} \\
    &= \left( \frac{1}{t-1}\sum_{s=1}^{t-1} \bar{\bf u}_{A_s}^{(t)} \bar{\bf u}_{A_s}^{{(t)}\top} +
      \frac{\lambda}{t-1} I \right)^{-1} \frac{1}{t-1}\sum_{s=1}^{t-1}
    \epsilon_{A_s}^{(t)} \bar{\bf u}_{A_s}^{(t)} \\
    &= \left( \sum_{a \in \mathcal{A}} \bar{\bf u}_{a}^{(t)} \bar{\bf u}_{a}^{{(t)}\top}
      \frac{\sum_{s=1}^{t-1} \indic{A_s = a}}{t-1} +
      \frac{\lambda}{t-1} I \right)^{-1} \sum_{a \in \mathcal{A}}
    \epsilon_{a}^{(t)} \bar{\bf u}_{a}^{(t)} \frac{\sum_{s=1}^{t-1} \indic{A_s = a}}{t-1} \\
    &= \left( \sum_{a \in \mathcal{A}} \bar{\bf u}_{a}^{(t)} \bar{\bf u}_{a}^{{(t)}\top} f_a(t-1) +
      \frac{\lambda}{t-1}I \right)^{-1} \sum_{a \in \mathcal{A}}
    \epsilon_{a}^{(t)} \bar{\bf u}_{a}^{(t)} f_a(t-1),
  \end{align*}
  where $f_a(t-1)$ is the empirical frequency with which action
  $a \in \mathcal{A}$ has been played up to and including time
  $t-1$. This allows us to equivalently interpret $z_{t-1}$ as the
  solution of a {\em weighted} $\ell^2$-regularized least squares
  regression problem with $K = |\mathcal{A}|$ observations (instead of
  the original interpretation with $t-1$ observations) as follows (we
  suppress the dependence of $f_a$ on $t$ as per the context for
  clarity of notation).

  Let $\mathbf{F}^{1/2}$ be the $A \times A$ diagonal matrix with the
  values $\sqrt{f_1}, \ldots, \sqrt{f_A}$ on the diagonal (note:
  $\sum_{a=1}^A f_a = 1$). With this, we can express $z_{t-1}$ as
  \begin{align*}
    z_{t-1}^{(t)} &= \arg\min_{z \in \mathbb{R}^C} \norm{\mathbf{F}^{1/2}
      \bar{U}^{(t)} z - \mathbf{F}^{1/2} \epsilon^{(t)}}_2^2 +
    \frac{\lambda}{t-1}\norm{z}_2^2 \\
    &= \arg\min_{z \in \mathbb{R}^C} \norm{\mathbf{F}^{1/2}\left(
        \bar{U}^{(t)} z - \epsilon^{(t)} \right) }_2^2 +
    \frac{\lambda}{t-1}\norm{z}_2^2 \\
    &= \arg\min_{z \in \mathbb{R}^C} \norm{ \left[ \begin{array}{cc}
          \mathbf{F}^{1/2} & 0 \\
          0 & \sqrt{\frac{\lambda}{t-1}} I_{C} \end{array} \right]
      \left( \left[ \begin{array}{c}
            \bar{U}^{(t)}\\
            I_{C} \end{array} \right] z - \left[ \begin{array}{c}
            \epsilon^{(t)} \\
            0 \end{array} \right] \right) }_2^2 \\
    &\equiv \arg\min_{z \in \mathbb{R}^C} \norm{\mathbf{D}^{1/2}
      \left( \mathbf{A} z - \mathbf{b} \right) }_2^2 =
    (\mathbf{A}^\top \mathbf{D} \mathbf{A})^{-1} \mathbf{A}^\top
    \mathbf{D} \mathbf{b},
  \end{align*}
  with $\mathbf{D}^{1/2}$ being a $(A+C) \times (A+C)$ diagonal \&
  positive semidefinite matrix,
  $\mathbf{A}^\top \mathbf{D} \mathbf{A} = \sum_{a \in \mathcal{A}}
  \bar{\bf u}_{a}^{(t)} \bar{\bf u}_{a}^{{(t)}\top} f_a(t-1) +
  \frac{\lambda}{t-1}I$
  being positive definite, and $\mathbf{A}$ having full column rank
  $C$. A result of \citet[Corollary 2.3]{forsgren1996linear} now gives
  \begin{align*}
    \norm{(\mathbf{A}^\top \mathbf{D} \mathbf{A})^{-1} \mathbf{A}^\top
      \mathbf{D}}_2 &\leq \max_{J} \norm{\mathbf{A}_J^{-1}}_2
  \end{align*}
  where $J$ ranges over all subsets of full-rank rows of $\mathbf{A}$,
  and $\mathbf{A}_J$ is the $C \times C$ submatrix of $\mathbf{A}$
  formed by picking rows $J$. Thus, $\norm{z_{t-1}^{(t)}}_2 \leq \left(
    \max_{J} \norm{\mathbf{A}_J^{-1}}_2 \right)
  \norm{\epsilon^{(t)}}_2$. This proves the lemma.
  \end{proof}

\section{Unregularized Least squares}\label{sec:Unregularized}

In our setting where we consider finitely many arms, one way wonder
whether it is possible to remove the regularization parameter
$\lambda$. Following \citet{rusmevichientong2010linearly}, this is
indeed possible under the assumption that the minimum eigenvalue of
$\sum_{a\in\cA} {\bf u}_a{\bf u}_a^\top$ is away from $0$. Then, we first play each arm once (once for all users $B$, not for
each of them) before running
Algorithm~\ref{alg:PerOFULwithExploration}, where \OFUL\ is used with
$\lambda =0$ and with $D_{t-1}$ redefined to be
$4R^2\bigg(A\log(t)+\log(A/\delta)\bigg)$.  This leads essentially to
similar bounds, with $\alpha^\star$ replaced by
$\max_J || U_J^{-1}||_2$, as we show below.

Let $\cU\subset \Real^C$.
We receive at time $s$, observation $y_s = {\bf u}_s^\top {\bf v}^\star + \eta_s \in \Real$ where ${\bf v}^\star\in\Real^C$ and ${\bf u}_s\in\cU$.

We make the following
\begin{assumption}\label{ass:main} 
There exists $R_\cX,R,\lambda_0 \in \Real^+_\star$ such that
\begin{enumerate}
\item $\forall s, ||{\bf u}_s|| \leq R_\cX$
\item $\forall \lambda \in\Real,\,\,\log\Esp\exp(\lambda \eta_s) \leq \lambda^2R^2/2$.
\item $\lambda_{\min}(\sum_{s=1}^t {\bf u}_s{\bf u}_s^\top) \geq \lambda_0$.
\end{enumerate}
\end{assumption}

Assumption~\ref{ass:main}.3 is satisfied for instance
when there are $C$ points $({\bf u}_{0,i})_{i\in [C]}$ in $\Real^d$ such that $\lambda_{\min}(\sum_{i=1}^C {\bf u}_{0,i}{\bf u}_{0,i}^\top)=\lambda_0>0$,
and ${\bf u}_s={\bf u}_{0,s}$ for $s\in[C]$. We consider the least-squares estimate
\beqan
{\bf v}_t = \Big(\sum_{s=1}^t {\bf u}_s{\bf u}_s^\top\Big)^{-1}\sum_{s=1}^\top {\bf u}_sy_s\,,
\eeqan

\subsection{Preliminary}
In case $\cU$ is finite, one can get the following result
\begin{mytheorem}\label{thm:finitecase}
Let us introduce the confidence set 
\beqan
\cC_t = \bigg\{ w\in \Real^C : w^\top G_t w \leq D_{t,\delta}\bigg\}\,,
\text{ where } G_t = \sum_{s=1}^t {\bf u}_s{\bf u}_s^\top
\eeqan
\beqan
\text{ and }\quad D_{t,\delta} =  4R^2\Big(|\cU|\log(t) +\log(|\cU|/\delta)\Big)\,.
\eeqan

Then, under Assumption~\ref{ass:main}, it holds
\beqan
\Pr\bigg({\bf v}_t -{\bf v}^\star \in \cC_t\bigg) \geq 1-\delta\,.
\eeqan
\end{mytheorem}

In the general case, it holds
\begin{mytheorem}\label{thm:generalcase}
Let us introduce the confidence set 
\beqan
\cC_t = \bigg\{ w\in \Real^C : w^\top G_t w \leq D_{t,\delta}\bigg\}\,,
\text{ where } G_t = \sum_{s=1}^t {\bf u}_s{\bf u}_s^\top
\eeqan
\beqan
\text{ and }\quad D_{t,\delta} =  16R^2\bigg[1+\log\left(1+\frac{36R_\cX^2}{\lambda_0}\right)\bigg]\bigg[C\log\left(\frac{36R_\cX^2}{\lambda_0}t\right)+\log(1/\delta) \bigg]\log(t)\,.
\eeqan

Then, under Assumption~\ref{ass:main}, and if $t\geq \frac{\lambda_0}{12 R_\cX^2}$ it holds
\beqan
\Pr\bigg({\bf v}_t -{\bf v}^\star \in \cC_t\bigg) \geq 1-\delta\,.
\eeqan
\end{mytheorem}

{\bf Proof:}
Indeed, let $z_t = \sum_{s=1}^\top {\bf u}_s\eta_s$.
Since $G_t$ is invertible, it holds that ${\bf v}_t = {\bf v}_\star + G_t^{-1}z_t$, and thus
\beqan
({\bf v}_t -{\bf v}^\star)^\top G_t ({\bf v}_t -{\bf v}^\star)
&=& z_tG_t^{-1}z_t
\eeqan
In the case when $\cU$ is finite, 
using the Proof of Theorem B.1 in \citet{rusmevichientong2010linearly}
then we further get 
for all $\epsilon>0$, 
\beqan
\Pr\bigg(z_tG_t^{-1}z_t \geq \epsilon^2 R^2 \bigg) 
\leq |\cU|t^{|\cU|}e^{-\epsilon^2/4}\,,
\eeqan
Thus, choosing $\epsilon = 2\sqrt{\log(|\cU|t^{|\cU|}/\delta)} $, we obtain that
\beqan
\Pr\bigg(z_tG_t^{-1}z_t \geq 4R^2\Big(|\cU|\log(t) +\log(|\cU|/\delta)\Big)\bigg) 
\leq \delta\,,
\eeqan 
which concludes the proof of Theorem~\ref{thm:finitecase}.

From the Proof of Theorem B.2 in \citet{rusmevichientong2010linearly}, it holds that
for all $\epsilon>2$,
\beqan
\Pr\bigg(z_tG_t^{-1}z_t \geq \epsilon^2 k_0^2R^2 \log(t)\bigg) 
\leq \big(36R_\cX^2 t / \lambda_0\big)^Ce^{-\epsilon^2/4}\,,
\eeqan
where $k_0=2\sqrt{1+\log(1+36R_\cX^2/\lambda_0)}$, which leads to
\beqan
\Pr\bigg(z_tG_t^{-1}z_t \geq  4\big(1+\log(1+36R_\cX^2/\lambda_0))R^2 \log(t) \epsilon^2\bigg) 
\leq \big(36R_\cX^2 t /\lambda_0\big)^Ce^{-\epsilon^2/4}\,,
\eeqan
Thus, let us use $\epsilon = 2\sqrt{\log\big(\big(36R_\cX^2 t /\lambda_0\big)^C/\delta\big)}$, 
which satisfies $\epsilon>2$ as soon as
$t>\frac{\lambda_0 e^{1/C}}{36 R_\cX^2}$, thus in particular
if $t\geq \frac{\lambda_0}{12 R_\cX^2}$.
Now, introducing the constant $c= 36R_\cX^2/\lambda_0$, we obtain
\beqan
\Pr\bigg(z_tG_t^{-1}z_t \geq  16R^2(1+\log(1+c))\log(t) \Big(C\log(ct)+\log(1/\delta)\Big)\bigg) 
\leq \delta\,,
\eeqan
which concludes the proof of theorem~\ref{thm:generalcase}.
 $\square$

 \subsection{Application to Low-Rank bandits}
 
 In order to apply this result to the low-rank bandit problem,
 we need to show that 
 $G_t$ is invertible.
 In our case, this matrix is at mini-session $n$ 
 $\tilde M_t = \sum_{s=1}^t \tilde {\bf u}_{n,a_s}\tilde  {\bf u}_{n,a_s}^\top$.
 
 Let us assume that all actions are sample at least once
 in the beginning. Thus, in this case
 $\lambda_{\min}(\tilde M_t)
 \geq \lambda_{\min}(\tilde A)$,
 where $\tilde A = \sum_{a\in [A]} \tilde  {\bf u}_{n,a}\tilde  {\bf u}_{n,a}^\top$.
 For convenience, let us also introduce the $C\times C$ matrix
 $A = \sum_{a\in [A]}  {\bf u}_{a} {\bf u}_{a}^\top = U^\top U$.
 
 In order to show that $\tilde M_t$ is invertible, it us enough to show that $\lambda_{\min}(\tilde A)>0$.
 
 Now, by the result of reconstruction of the feature matrix $M$, we know that there exists with high probability
 a permutation $\pi$  such that the columns are well estimated:
 \beqan
  \forall c, || {\bf u}_{\pi(c)} - \tilde  {\bf u}_{n,c}|| \leq 
 \Diamond 
 	A^{3}\sqrt{\sum_{i=1}^n\gamma_i^{-2}\frac{C\log(4A^3/\delta)}{2n^2}}\,.
 	\eeqan
 
 Thus, we study $E = \tilde A - A$. Let $\lambda$ be any eigenvalue of $E$, then it holds 
 \beqan
 \lambda &\leq& \text{trace}(E) =\sum_{a\in[A]}\text{trace}\bigg(
 \tilde  {\bf u}_{n,a}\tilde  {\bf u}_{n,a}^\top -  {\bf u}_{a} {\bf u}_{a}^\top\bigg)\\
 &\leq& \sum_{a\in[A]} ||\tilde  {\bf u}_{n,a}||^2
 - ||  {\bf u}_{a}||^2\\
 &\leq& \sum_{a\in[A]} \sum_{c\in[C]} 
 \tilde u_{n,a,c}^2 - u_{a,c}^2 \\
 &\leq&
 \sum_{a\in[A]} \sum_{c\in[C]} (\tilde u_{n,a,c} - u_{a,c})^2
 + 2u_{a,c}(\tilde u_{n,a,c} - u_{a,c})\\
 &\leq&\sum_{c\in[C]} ||\tilde u_{n,c} - u_{c}||^2
 + 2 \sum_{c\in[C]} \sqrt{\sum_{a\in[A]}u_{a,c}^2}
 \sqrt{\sum_{a\in[A]} (\tilde u_{n,a,c} - u_{a,c})^2}\\
 &\leq&\sum_{c\in[C]} ||\tilde  {\bf u}_{n,c} -  {\bf u}_{c}||^2
 +2 || {\bf u}_c|| ||\tilde  {\bf u}_{n,c} -  {\bf u}_{c}||
 \\
 &\leq&(2u_{max}+1)\sum_{c\in[C]} ||\tilde  {\bf u}_{n,c} -  {\bf u}_{c}|| \,.
 \eeqan
 Thus, provided that $n$ is large enough that 
 \beqan
 \lambda_{\min}(A) >2(2u_{max}+1)\sum_{c\in[C]} ||\tilde  {\bf u}_{n,c} -  {\bf u}_{c}||\,,
 \eeqan
 we deduce that $\tilde M_t$ is invertible.
  Using the fact that $A= U^\top U$,
 This translates to the condition
 \beqan
  \lambda_{\min}(U^\top U)>
 2\Diamond 
 	(2u_{max}+1)CA^{3}\sqrt{\sum_{i=1}^n\gamma_i^{-2}\frac{C\log(4A^3/\delta)}{2n^2}}\,
 \eeqan
 that is
 \beqan
  \frac{n^2}{\sum_{i=1}^n \gamma_i^{-2}}>   
  	\frac{4\Diamond^2(2u_{max}+1)^2C^3A^{6}\log(4A^3/\delta)}{\lambda^2_{\min}(U^\top U)}\,.
 \eeqan
 
 Thus, assuming that all actions are chosen at least once in the beginning, and that 
  \beqan
   \frac{n^2}{\sum_{m=1}^n \gamma_m^{-2}}>   
   	\frac{4\Diamond^2(2u_{max}+1)^2C^3A^{6}\log(4A^3/\delta)}{\lambda^2_{\min}(U^\top U)}\,,
  \eeqan
  then $\lambda_{\min}(\tilde M_t) \geq \lambda_{\min}(U^\top U)/2 = \lambda_0/2>0$ and   Theorem~\ref{thm:finitecase} and Theorem~\ref{thm:generalcase} both apply.

 In order to control the regret of the unregularized version of \OFUL, we now 
 use the proof of \citet[Theorem 4.1]{rusmevichientong2010linearly} combined with the fact that $\lambda_{\min}(\tilde M_t)\geq \lambda_0/2$
 to get 
 \beqan
 \sum_{t=A+1}^n \min\{||\bar{\bf u}_{A_{t}}||^2_{\tilde M_{t-1}^{-1}},1\} \leq2\max\{1, \frac{2R_\cX^2}{\lambda_0}\}\bigg(C\log(\max\{1, \frac{2R_\cX^2}{\lambda_0}\}) + (C+1)\log(n+1)\bigg)\,.
 \eeqan
 
 A straightforward adaptation of the proof of Theorem~\ref{thm:pOFUL2} then gives
 \beqan
 \kR_n &\leq& \rho'\sqrt{n(A+16 D_{n,\delta}^2 \sum_{t=A+1}^n \min\{||\bar{\bf u}_{A_{t}}||^2_{\tilde M_{t-1}^{-1}},1\})}\\
 &\leq& 16\rho'R^2\bigg(A\log(n)+\log(A/\delta)\bigg)\sqrt{n
 \Big(2+\frac{4R_\cX^2}{\lambda_0}\Big)\bigg(C\log\Big(1\hspace{-1mm}+\hspace{-1mm}\frac{2R_\cX^2}{\lambda_0}\Big) + (C\hspace{-1mm}+\hspace{-1mm}1)\log(n\hspace{-1mm}+\hspace{-1mm}1)\bigg)}\\
 && + \rho'\sqrt{An}\,.
 \eeqan
 
 Following the same steps as for Lemma~\ref{lem:pOFULSetting}, we finally obtain the result:
  \begin{mytheorem}[Unregularized OFUL robustness result]
    \label{thm:UpOFUL2}
    Assume $||{\bf v}^\circ||_2\leq R_\Theta$, for all $a\in\cA$,
    $||\bar{\bf u}_a||_2 \leq R_\cX$ and $|m_a| \leq 1$, and that
    $\arg \max_{a \in \cA} \bar{\bf u}_a^\top {\bf v}^\circ = \{a^\star\}$ (i.e.,
    the linearly realizable approximation has $a^\star$ as its unique
    optimal action).    Assume that each action has been played once.  Let $0 < \delta \leq 1$. 
    Provided that the number of mini-sessions $n_0$ is
     large enough to satisfy 
    \beqan
        \frac{n_0^2}{\sum_{i=1}^{n_0}\gamma_i^{-2}} &\geq& \tilde \varhexagon_{b,\delta}
    %
    \eeqan
    where 
    \beqan
     \tilde \varhexagon_{b,\delta} &=&      \max\bigg\{ \frac{2A^6\log(4A^2/\delta)}{\min\{\Gamma,\sigma_{\min}\}^2},
          \frac{A^9(1 + 10(\frac{1}{\Gamma} + \frac{1}{\sigma_{\min}})(1+u_{\max}^3))^2C^{5}\log(4A^3/\delta)}{2C^2_1\sigma_{min}^{3}}\\
          &&
          \frac{4\Diamond^2(2u_{max}+1)^2C^3A^{6}\log(4A^3/\delta)}{\lambda^2_{\min}(U^\top U)},\\
          &&\Diamond^2 A^6C^2\log(4A^3/\delta)\max\bigg\{
            2\ol\alpha_\star^2,
            \frac{8A||{\bf v}_b||_2^2}{g_b^2},
            \frac{2^7\ol\alpha_\star^2 Cu_{\max}^2 ||{\bf v}_b||_2^2}{g_b^2}
            +1/2
            \bigg\}
          \bigg\},
    \eeqan   
   then with   probability at least $1 - \delta$ for all $T \geq 0$,
   the regret $\kR_{A+1:n}$ of the \OFUL\ algorithm from decision $A+1$ to $n$ satisfies
    \[ \kR_{A+1:n} \leq 32R^2\bigg[A\log(n)+\log(A/\delta)\bigg]\sqrt{n
     \Big(2+\frac{4\ol R_\cX^2}{\lambda_0}\Big)\bigg(C\log\Big(1\hspace{-1mm}+\hspace{-1mm}\frac{2\ol R_\cX^2}{\lambda_0}\Big) + (C\hspace{-1mm}+\hspace{-1mm}1)\log(n\hspace{-1mm}+\hspace{-1mm}1)\bigg)}, \]
  where we introduced
      \beqan
       \ol R_\cX =
        \max_{a\in\cA}||{\bf u}_a||_2 + \frac{\sqrt{A}}{2\ol \alpha_\star} \quad\text{and}\quad \ol\alpha_\star = \min_J || U_{J}^{-1}||\,.\eeqan
  \end{mytheorem}
  
This result enables to get the corresponding variant of Theorem~\ref{thm:main} using an unregularized \OFUL.

\end{document}